\documentclass{article}

\usepackage[final]{neurips_2021}

\usepackage[utf8]{inputenc} 
\usepackage[T1]{fontenc}    
\usepackage{hyperref}       
\usepackage{url}            
\usepackage{booktabs}       
\usepackage{xcolor}
\usepackage{amsfonts}       
\usepackage{nicefrac}       
\usepackage{microtype}      
\usepackage{appendix}
\usepackage{amsthm}
\usepackage{amsmath}
\usepackage{amssymb}
\usepackage{amsfonts}
\usepackage{enumerate}
\usepackage{algorithm}
\usepackage{algorithmic} 
\usepackage{geometry}
\usepackage{xspace}
\usepackage{mathabx}
\usepackage{graphicx}
\usepackage{times}
\usepackage{tabularx, makecell, booktabs}

\usepackage{verbatim}

\hypersetup{
  colorlinks   = true, 
  urlcolor     = blue, 
  linkcolor    = blue, 
  citecolor   = blue 
}
\newtheorem{theorem}{Theorem}

\newtheorem{assum}[theorem]{Assumption}
\newtheorem{claim}[theorem]{Claim}
\newtheorem{lemma}[theorem]{Lemma}

\newtheorem{remark}{Remark}

\usepackage{amsmath,amsfonts,bm,physics,bbm,mathrsfs}

\def\clip{{\mathsf{clip}}}
\def\pj{\mathsf{clip}}

\def\vmu{{\bm{\mu}}}
\def\ind{{I}}

\def\cvx{\check{\bm x}}
\def\cgT{\check{\mathcal T}}
\def\ceta{\check{\eta}}
\def\czeta{\check{\zeta}}
\def\cR{\check{R}}
\def\cM{\check{M}}
\def\Reg{\mathfrak{R}}
\def\cReg{\check{\mathfrak{R}}}

\def\1{\bm{1}}

\def\vmu{{\bm{\mu}}}
\def\vtheta{{\bm{\theta}}}

\def\ve{{\bm{e}}}

\def\vu{{\bm{u}}}
\def\vv{{\bm{v}}}

\def\vx{{\bm{x}}}

\def\mI{{\bm{I}}}

\def\mU{{\bm{U}}}

\DeclareMathAlphabet{\mathsfit}{\encodingdefault}{\sfdefault}{m}{sl}
\SetMathAlphabet{\mathsfit}{bold}{\encodingdefault}{\sfdefault}{bx}{n}

\def\gA{{\mathcal{A}}}
\def\gB{{\mathcal{B}}}

\def\gE{{\mathcal{E}}}
\def\gF{{\mathcal{F}}}

\def\gS{{\mathcal{S}}}
\def\gT{{\mathcal{T}}}

\def\sB{{\mathbb{B}}}

\def\sI{{\mathbb{I}}}

\def\sN{{\mathbb{N}}}

\def\sR{{\mathbb{R}}}

\def\sZ{{\mathbb{Z}}}

\newcommand{\E}{\mathbb{E}}

\newcommand{\Var}{\mathrm{Var}}

\renewcommand{\E}{\mathop{\mathbb E}}

\DeclareMathOperator*{\poly}{poly}
\DeclareMathOperator*{\polylog}{polylog}

\DeclareMathOperator*{\argmax}{arg\,max}

\title{
Improved Variance-Aware Confidence Sets for \\ Linear Bandits and Linear Mixture MDP
}

\author{
Zihan Zhang\thanks{Equal contribution.} \\
Tsinghua University \\
\texttt{zihan-zh17@mails.tsinghua.edu.cn} \\
\\
\And
	Jiaqi Yang\footnotemark[1]\\
Tsinghua University\\
\texttt{yangjq17@gmail.com} \\
\\
\AND
Xiangyang Ji \\
Tsinghua University \\
\texttt{xyji@tsinghua.edu.cn}
\And
Simon S. Du \\ 
University of Washington \\
\texttt{ssdu@cs.washington.edu}
}

\begin{document}

\maketitle

\begin{abstract}\label{abs}
This paper presents new \emph{variance-aware} confidence sets for linear bandits and linear mixture Markov Decision Processes (MDPs).
With the new confidence sets, we obtain the follow regret bounds:
\begin{itemize}
\item For linear bandits, we obtain an $\widetilde{O}(\poly(d)\sqrt{1 + \sum_{k=1}^{K}\sigma_k^2})$ data-dependent regret bound, where $d$ is the feature dimension, $K$ is the number of rounds, and $\sigma_k^2$ is the \emph{unknown} variance of the reward at the $k$-th round. This is the first regret bound that only scales with the variance and the dimension but \emph{no explicit polynomial dependency on $K$}.
When variances are small, this bound can be significantly smaller than the $\widetilde{\Theta}\left(d\sqrt{K}\right)$ worst-case regret bound.
\item For linear mixture MDPs, we obtain an $\widetilde{O}(\poly(d, \log H)\sqrt{K})$ regret bound, where $d$ is the number of base models, $K$ is the number of episodes, and $H$ is the planning horizon. 
This is the first regret bound that only scales \emph{logarithmically} with $H$ in the reinforcement learning with linear function approximation setting, thus \emph{exponentially improving} existing results, and resolving an open problem in \citep{zhou2020nearly}.
\end{itemize}
We develop three technical ideas that may be of independent interest:
1) applications of the peeling technique to both the input norm and the variance magnitude, 2) a recursion-based estimator for the variance, and 3) a new convex potential lemma that generalizes the seminal elliptical potential lemma.
\end{abstract}

\section{Introduction}\label{sec:intro}

In sequential decision-making problems such as bandits and reinforcement learning (RL), the agent chooses an action based on the current state, with the goal to maximize the total reward.
When the state-action space is large, function approximation is often used for generalization.
One of the most fundamental and widely used methods is linear function approximation.

  For (infinite-actioned) linear bandits, the minimax-optimal regret bound is $\widetilde{\Theta}(d\sqrt{K})$~\citep{dani2008Stochastic,abbasi2011improved}, where $d$ is the feature dimension and $K$ is the number of total rounds played by the agent.\footnote{We follow the reinforcement learning convention to use $K$ to denote the total number of rounds / episodes.}
However, oftentimes the worst-case analysis is overly pessimistic, and it is possible to obtain data-dependent bound that is substantially smaller than $\widetilde{O}(d\sqrt{K})$ in benign scenarios.

One direction to study is the variance magnitude.
As a motivating example, in linear bandits, if there is no noise (variance is $0$), one only needs to pay at most $d$ regret to identify the best action because $d$ samples are sufficient to recover the underlying linear coefficients (in general position).
This constant-type regret bound is much smaller than the $\sqrt{K}$-type regret bound in the worst case where the variance magnitude is a lower bounded constant.
Therefore, a natural question is:
\begin{center}
\textbf{Can we design an algorithm that  adapts to the variance magnitude, and its regret degrades gracefully from the benign noiseless constant-type bound to the worst-case  $\sqrt{K}$-type bound?}
\end{center}

In RL, exploiting the variance information is also important.
For tabular RL,  one needs to utilize the variance information, e.g., Bernstein-type exploration bonus to achieve the minimax optimal regret~\citep{azar2017minimax,zanette2019tighter,zhang2020model,zhang2020reinforcement,menard2021ucb,dann2019policy}.
For example, the recently proposed MVP algorithm~\citep{zhang2020reinforcement}, enjoys an $\widetilde{O}(\polylog(H)\times (\sqrt{SAK}+S^2 A))$ regret bound, where $S$ is the number of states, $A$ is the number of actions, $H$ is the planning horizon, and $K$ is the total number of episodes. \footnote{$\widetilde{O}(\cdot)$ hides logarithmic factors. Sometimes we write out $\polylog H$ explicitly to emphasize the logarithmic dependency on $H$.}\footnote{
	This bound holds for setting where the transition is homogeneous and the total reward is bounded by $1$. We focus on this setting in this paper. See Section~\ref{sec:rel} and~\ref{sec:pre} for more discussions.
} 
Notably, this regret bound only scales \emph{logarithmically} with $H$.
On the other hand, without using the variance information, e.g., using Hoeffding-type bonus instead of Bernstein-type bonus, algorithms would suffer a regret that scales \emph{polynomially} with $H$~\citep{azar2017minimax}.

Going beyond tabular RL, a recent line of work studied RL with linear function approximation with different assumptions~\citep{yang2019sample,modi2020sample,jin2019provably,ayoub2020model,zhou2020nearly,modi2020sample}.
Our paper studies the linear mixture Markov Decision Process (MDP) setting~\citep{modi2020sample,ayoub2020model,zhou2020nearly}, where the transition probability can be represented by a linear function of some features or  base models.
This model-based assumption is motivated by problems in robotics and queuing systems. We refer readers to \cite{ayoub2020model} for more discussions.

For this linear mixture MDP setting, previous works can obtain regret bounds in the form $\widetilde{O}(\poly
(d,H)\sqrt{K})$, where $d$ is the number of base models.
While these bounds do not scale with $SA$, they scale \emph{polynomially} with $H$, because the algorithms in previous works do not use the variance information.
In practice, $H$ is often large, and even a polynomial dependency on $H$ may not be acceptable.
Therefore, a natural question is 
\begin{center}
\textbf{Can we design an algorithm that exploits the variance information to obtain an $\widetilde{O}(\poly(d, \log H)\sqrt{K})$ regret bound for linear mixture MDP?}
\end{center}

\subsection{Our Contributions}
In this paper, we develop new, \emph{variance-aware} confidence sets for linear bandits and linear mixture MDP and answer the above two questions affirmatively. 

\paragraph{Linear Bandits.} For linear bandits, we obtain an $\widetilde{O}(\poly(d)\sqrt{1 + \sum_{k=1}^{K}\sigma_k^2})$ regret bound, where $\sigma_k^2$ is the \emph{unknown} variance at the $k$-th round.
	To our knowledge, this is the first bound that solely depends on the variance and the feature dimension, and has no explicit polynomial dependency on $K$. When the variance is very small so that $\sigma_k^2 \ll 1$, this bound is substantially smaller than the worst-case $\widetilde{\Theta}(d \sqrt K)$ bound.
Furthermore, this regret bound naturally interpolates between the worst-case $\sqrt{K}$-type bound and the noiseless-case constant-type bound.
	
\paragraph{Linear Mixture MDP.} For linear mixture MDP, we obtain the desired $\widetilde{O}(\poly(d, \log H)\sqrt{K})$ regret bound.
	This is the first regret bound in RL with function approximation that 1) does not scale with the size of the state-action space, and 2) only scales \emph{logarithmically} with the planning horizon $H$.
	Therefore, we exponentially improve existing results on RL with linear function approximation in term of the $H$ dependency, and resolve an open problem in \citep{zhou2020nearly}.
	More importantly, our result conveys the positive conceptual message for RL: it is possible to simultaneously overcome the two central challenges in RL, \textit{large state-action space} and \textit{long planning horizon}.

\subsection{Main Difficulties and Technical Innovations}
We first describe limitations of existing works why they cannot achieve the desired regret bounds described above.

\paragraph{Limitations of Existing Variance-Aware Confidence Sets}
\citet{faury2020improved,zhou2020nearly} applied Bernstein-style inequalities to construct a confidence sets of the least square estimator for linear bandits.
However, their methods can not be applied directly to obtain the desired data-dependent regret bound. \citet{abeille2021instance} also designed an variance-dependent confidence set for logistic bandits. However in their problem the rewards are Bernoulli and the variance is a function of the mean.



We give a simple example to illustrate their limitations.
Consider the case where the variance is always $\sigma^2 \ll 1$.
Let $\left(\vx_1,y_1\right),\ldots,\left(\vx_{k-1},y_{k-1}\right)$ be the samples collected before the $k$-th round.
 Their confidence set at the $k$-th round is 
$\Theta_{k} =\{ \vtheta| ||\vtheta-\hat{\vtheta}_k ||_{\Lambda_{k-1}} \leq C(\sigma\sqrt{d}+1+\lambda^{1/2})  \}$ (See In  Equation~(4.3) of \cite{zhou2020nearly} and Theorem~1 of \cite{faury2020improved}).
 where  $\Lambda_{k-1}=\sum_{\tau=1}^{k-1}\vx_\tau\vx_\tau^{\top}+\lambda I$ is the un-normalized covariance matrix , $\hat{\vtheta}_k = \Lambda_{k-1}^{-1} \sum_{\tau=1}^{k-1}y_{\tau}\vx_{\tau}$ is the estimated linear coefficients by least squares, $\lambda$ is a regularization parameter and $C$ is a constant. 
Consider the case $d=1$ and $\vx_{k}=\sqrt{1/K}$ for $k=1,\ldots,K$. 
Their regret bound is roughly \[\sum_{k=1}^K (\sigma \sqrt{d}+1+\lambda^{1/2}) \|\vx_{k}\|_{\Lambda_{k}^{-1}}\geq (1+\lambda^{1/2})\sum_{i=1}^K \|\vx_{k}\|_{\Lambda_{k}^{-1}} \geq (1+\lambda^{1/2})\sqrt{\frac{K}{1+\lambda}} \geq \sqrt{K},\] which is much larger than our bound, $O\left(\sqrt{K\sigma^2+1}\right)$ when $\sigma$ is very small.
For more detailed discussion, please refer to Appendix~\ref{app:diff}.

Below we describe our main techniques.
\paragraph{Elimination with Peeling.}
Instead of using least squares and upper-confidence-bound (UCB), we use an elimination approach.
%
More precisely, for the underlying linear coefficients $\vtheta^* \in \mathbb{R}^d$, we build a confidence interval for $\left(\vtheta^*\right)^\top \vmu$ for every $\vmu$ in an $\epsilon$-net of the $d$-dimensional unit ball, and we eliminate $\vtheta \in \mathbb{R}^d$ if $\vtheta^\top\vmu$ fails to fall in the confidence interval of $(\vtheta^*)^{\top}\vmu$ for some $\vmu$. 
To build the confidence intervals, we use 1) an empirical Bernstein inequality (cf. Theorem~\ref{thm:m0m1}) and 2) the peeling technique to both the input norm and the variance magnitude.
As will be clear in the proof (cf. Section~\ref{sec:proof_bandit}), this peeling step is crucial to obtain a tight regret bound for the example above.
The new confidence region provides a tighter estimation for $\vtheta^*$, which helps address the drawback in least squares.

\paragraph{Generalization of the Elliptical Potential Lemma.}
Since we use the peeling technique which comes with a clipping operation, we cannot use the seminal elliptic potential lemma~\cite{dani2008Stochastic} any more.
Instead, we propose a more general lemma below, which provides a bound of potential for a general class of convex functions though with a worse dependency on $d$ than the bound in the elliptical potential lemma.
We believe this lemma can be applied to other problems as well.
\begin{lemma}[Generalized Quadratic  Potential Lemma]\label{lemma:gpl}  Let $f(x)\geq 0$ be a convex function over $\mathbb{R}$ such that $ \frac{f(x)}{x^2} \leq \frac{f(y)}{y^2}\leq 1$ and $f(x)\geq f(y)$ if $x^2\geq y^2>0$. Let $ \mathbb{B}(1)$ denote the $d$-dimensional unit ball.  Fix $\ell\in (0,1]$. For any $\vx_1,\vx_2,\ldots, \vx_{t}\in \mathbb{B}(1)$  and $\vmu_1,\vmu_2,\ldots,\vmu_t\in \mathbb{B}(1)$, we have that
\begin{align}
    \sum_{i=1}^{t} \min \left\{\frac{f(\vx_i \vmu_i)}{ \sum_{j=1}^{i-1} f(\vx_j \vmu_i) +\ell^2},1 \right\} \leq O(d^4\log(dt/\ell)).\nonumber
\end{align}
\end{lemma}
Note that by choosing $f(x)=  x^2$ and $\vmu_i =\frac{\vx_i \Lambda_{i}^{-1}}{\| \vx_i \Lambda_{i}^{-1}\|} $ with $\Lambda_{i}=\sum_{j=1}^{i-1}\vx_j\vx_j^{\top} +\ell\mathbf{I}$, Lemma~\ref{lemma:gpl} reduces to the classical elliptic potential lemma~\citep{dani2008Stochastic}. 
Our proof consists of two major parts. We first establish a symmetric version of Equation~\eqref{eq:gpl} using rearrangement inequality, and then bound the number of times the energy for some $\vmu$ (i.e., $\sum_{j=1}^{i}f(\vx_j\vmu)+l^2$) doubles. The full proof is deferred to Appendix~\ref{sec:gpl}.


For linear mixture MDP, we propose another technique to further reduce the dependency on $d$.
\paragraph{Recursion-based Variance Estimation.}
In linear bandits, generally it is not possible to estimate the variance because the variance at each round can arbitrarily different.
On the other hand, for linear mixture MDP, the variance is a quadratic function of the underlying coefficient $\vtheta^*$.
Furthermore, the higher moments are polynomial functions of $\vtheta^*$.
Utilizing this rich structure and leveraging the recursion idea in previous analyses on tabular RL~\citep{lattimore2012pac,li2020breaking,zhang2020reinforcement}, we explicitly estimate the variance and higher moments to further reduce the regret.
See Section~\ref{sec:rl} for more explanations.

\section{Related Work} \label{sec:rel}
\paragraph{Linear Bandits.}
There is a line of theoretical analyses of linear bandits problems~\citep{auer02nonstochastic,dani2008Stochastic,chu2011contextual,abbasi2011improved,li2019nearly,li2019tight}.
For infinite-actioned linear bandits, the minimax regret bound is $\widetilde{\Theta}(d\sqrt{K})$. and recent works tried to give fine-grained instance-dependent bounds~\citep{katz2020empirical,jedra2020optimal}. 
For multi-armed bandits, \citet{audibert2006use} showed by exploiting the variance information, one can improve the regret bound.
For linear bandits, only a few work studied how to use the variance information.
\citet{faury2020improved} studied logistic bandit problem with adaptivity to the variance of noise, where a Bernstein-style confidence set was proposed.
However, they assume the variance is known 
and cannot attain the desired variance-dependent bound due to the example we gave above. 
Linear bandits can be also seen as a simplified version of RL with linear function approximation, where the planning horizon degenerates to $H=1$.

\paragraph{RL with Linear Function Approximation.} Recently, it is a central topic in the theoretical RL community to figure out the necessary and sufficient conditions that permit efficient learning in RL with large state-action space \citep{wen2013efficient,jiang2017contextual,yang2019sample,yang2019reinforcement,du2019q,du2019good,du2019provably,du2020agnostic,jiang2017contextual,feng2020provably,sun2018model,dann2018oracle,krishnamurthy2016pac,misra2019kinematic,ayoub2020model,zanette2020learning,wang2019optimism,wang2020provably,wang2020reward,jin2019provably,weisz2020exponential,modi2020sample,shariff2020efficient,jin2019provably,cai2019provably,he2020logarithmic,zhou2020nearly}.
However, to our knowledge, all existing regret upper bounds have a polynomial dependency on the planning horizon $H$, except works that assume the environment is deterministic~\citep{wen2013efficient,du2020agnostic}. 

This paper studies the linear mixture MDP setting \citep{ayoub2020model,zhou2020provably,zhou2020nearly,modi2020sample}, which assumes the underlying transition is a linear combination of some known base models.
\citet{ayoub2020model} gave an algorithm, UCRL-VTR, with an $\widetilde{O}(dH^2\sqrt{K})$ regret in the time-inhomogeneous model.\footnote{
	The time-inhomogeneous model refers to the setting where the transition probability can vary at different levels, and the time-homogeneous model refers to the setting where the transition probability is the same at different levels.
Roughly speaking, the model complexity of the time-inhomogeneous model is $H$ times larger than that of the time-homogeneous model.
	In general, it is straightforward to tightly extend a result for the time-homogeneous model to the time-inhomogeneous model by extending the state-action space \citep[Footnote 2]{jin2018q}, but not vice versa. 
}
Our algorithm improves the $H$-dependency from $\poly(H)$ to $\polylog(H)$, at the cost of a worse dependency on $d$.

\paragraph{Variance Information  in Tabular MDP.}
The use of the variance information in tabular MDP was first proposed by \citet{lattimore2012pac} in the discounted MDP setting, and was later adopted in the episodic MDP setting~\citep{azar2017minimax,jin2018q,zanette2019tighter,dann2019policy,zhang2020reinforcement,zhang2020almost}.
This technique is crucial to tighten the dependency on $H$.

\paragraph{Concurrent Work by \citet{zhou2020nearly}.}
While preparing this draft, we noticed a concurrent work by \citet{zhou2020nearly}, who also studied how to use the variance information for linear bandits and linear mixture MDPs.
We first compare their results with ours.
For linear bandits, they proved an $\widetilde{O}(\sqrt{dK} + d\sqrt{\sum_{i=1}^K \sigma_i^2})$ regret bound, while we prove an $\widetilde{O}(d^{4.5} \sqrt{\sum_{i=1}^{K} \sigma_i^2} + d^5)$ regret bound.
Our bound has a worse dependency on $d$, but in the regime where $K$ is very large and the sum of the variances is small, our bound is stronger.
Furthermore, they assumed \emph{the variance is known while we do not need this assumption}.
For linear mixture MDP, they proved an  $\widetilde{O}(\sqrt{d^2H+dH^2}\sqrt{K}+d^2H^2+d^3H)$  bound for the time-inhomogeneous model, while we prove an $\widetilde{O}(d^{4.5} \sqrt{K} + d^5 )\times \polylog(H)$  bound for the time-homogeneous model.
Their bound has a better dependency on $d$ than ours and is near-optimal in the regime $K = \Omega\left(\poly\left(d,H\right)\right)$ and $H=O(d)$.
On the other hand, we have an exponentially better dependency on $H$ in the time-homogeneous model.
Indeed, obtaining a regret bound that is logarithmic in $H$ (in the time-homogeneous model) was raised as an open question in their paper \citep[Remark 5.5]{zhou2020nearly}.

Next, we compare the algorithms and the analyses.
The algorithms in the two papers are very different in nature:  ours are based on elimination while theirs are based on least squares and UCB.
We note that, for linear bandits, their current analysis cannot give a $\sqrt{K}$-free bound because there is a term that scales \emph{inversely} with the variance.
This can be seen by plugging the first line of their (B.25) to their (B.23).
For the same reason, they cannot give a horizon-free bound in the time-homogeneous linear mixture MDP. 
In sharp contrast, our analysis does not have the term depending on the inverse of the variance.
On the other hand, their algorithms are computationally efficient (given certain computation oracles), but our algorithms are not because ours are elimination-based.
See Section~\ref{sec:conclusion} for more discussions.

\section{Preliminaries}
\label{sec:pre}

\paragraph{Notations.} We use $\sB^d_{p}(r) = \{x \in \sR^d : \lVert x \rVert_{p} \le r\}$ to denote the $d$-dimensional $\ell_p$-ball of radius $r$, so $\mathbb{B}(1)= \mathbb{B}_2^d(1)$
For any set $S \subseteq \sR^d$, we use $\partial S$ to denote its boundary. For $N \in \sN,$ we define $[N] = \{1, \ldots, N\}.$
One important operation used in our algorithms and analyses is clipping.
Given $\ell>0$ and $u \in \mathbb{R} $, we define \[
\clip(u, \ell) = \min\{\lvert u \rvert, \ell \}\cdot \frac{u}{\lvert u \rvert}\] for $u\neq 0$ and $\clip(0,\ell)=0$. 
For any two vectors $\vu, \vv$, to save notations, we use $\vu \vv = \vu^\top \vv$ to denote their inner product when no ambiguity. 

\paragraph{Linear Bandits.} We use $K$ to denote the number of rounds in the linear bandits.
 At each round $k = 1, \ldots, K,$ the algorithm is first given the context set $\gA_k \subseteq \sB_2^d(1),$ then the algorithm chooses an action $\vx_k \in \gA_k$ and receives the noisy reward $r_k = \vx_k \vtheta^* + \varepsilon_k,$ where $\vtheta^* \in \sB^d_2(1)$ is the unknown underlying linear coefficients and $\varepsilon_k$ is the random noise. 
 We define $\gF_k = \sigma(\vx_1, \varepsilon_1, \ldots, \vx_k, \varepsilon_k, \vx_{k+1}).$ We assume that $\abs{r_k} \le 1$ and that the noise $\varepsilon_k$ satisfies $\E[\varepsilon_k \mid \gF_k] = 0$ and $\E[\varepsilon_k^2 \mid \gF_k] = \sigma_k^2.$ The goal is to learn $\vtheta^*$ and minimize the cumulative expected regret $\E[\Reg^K]$, where 
\begin{align*}
    \Reg^K = \sum_{k = 1}^K [\max_{\vx \in \gA_k} \vx \vtheta^* - \vx_k \vtheta^*].
\end{align*}

\begin{remark}
Here we assume the reward is uniformly bounded ($\abs{r_k} \le 1$) instead of $1$-sub-Gaussian commonly used in the literature only for the ease of presentation, because in RL, it is standard to assume bounded reward.
Note if the noise is $1$-sub-Gaussian, our algorithm also applies with only an $O\left(\log K\right)$ overhead because a problem with $1$-sub-Gaussian noise can be reduced to that with uniformly bounded noise by clipping the noise with a threshold $O(\log K )$.
\end{remark}

\paragraph{Episodic MDP and Linear Mixture MDP.} We use a  tuple $(\gS, \gA, r, P, K, H)$ to define an episodic finite-horizon MDP. Here, $\gS$ is its state space, $\gA$ is its action space, $r : \gS \times \gA \to [0, 1]$ is its reward function, $P(s' \mid s, a)$ is the transition probability from the state-action pair $(s, a)$ to the new state $s'$, $K$ is the number of episodes, and $H$ is the planning horizon of each episode. 
Without the loss of generality, we assume a fixed initial state $s_1$.
A sequence of functions $\pi = \{\pi_h : \gS \to \triangle(\gA)\}_{h=1}^H$ is an policy, where $\triangle(\gA)$ denotes the set of all possible distributions over $\gA$.

At each episode $k = 1, \ldots, K$, the algorithm outputs a policy $\pi^k$, which is then executed on the MDP by $a_h^k \sim \pi_h^k(s_{h}^k), s_{h + 1}^k \sim P(\cdot \mid s_h^k, a_h^k)$.
 We let $r_h^k = r(s_h^k, a_h^k)$ be the reward at time step $h$ in episode $k$. 
Importantly, we assume the transition model $P(\cdot \mid \cdot, \cdot)$ is time-homogeneous, which is necessary to bypass the $\poly(H)$ dependency. We assume that the reward function is known, which is standard in the theoretical RL literature to simplify the presentation~\citep{modi2020sample,ayoub2020model}.
We let $\pi^*$ to denote the optimal policy which achieves the maximum reward in expectation.

We make the following regularity assumption on the rewards: the sum of reward, $\sum_{h=1}^H r_h$, in each episode is bounded by $1$. 
\begin{assum}[Non-uniform reward] \label{assum1}
 $\sum_{h=1}^H r_h^k\le 1$ almost surely for any policy $\pi^k$.
\end{assum}
This assumption is much weaker than the common assumption where the reward at each time step is bounded by $1/H$ (uniform reward) because Assumption~\ref{assum1} allows one spiky reward as large as $\Omega\left(1\right)$. See more discussions about this reward scaling in \cite{jiang2018open,wang2020long,zhang2020reinforcement}.

For any policy $\pi$, we define its $H$-step $V$-function and $Q$-function as 
\begin{align*}
& V^{\pi}_h(s) = \max_{a \in \gA} Q^\pi_h(s, a) \\
\text{where~} &Q^\pi_h(s, a) = r(s, a) + \mathbb E_{s' \sim P(\cdot \mid s, a)} V_{h + 1}^\pi(s') \text{ for }h=1,\ldots, H
\end{align*}
where we set $V_{H+1} = 0$.
For simplicity, we also denote $V^{\pi}(s_1) = V^{\pi}_1(s_1)$ and $V^*(s_1) = V^{\pi^*}(s_1)$.

A linear mixture MDP is an episodic MDP with the extra assumption that its transition model is an unknown linear combination of a known set of models. Specifically, there is an unknown parameter $\vtheta^* \in \sB^d_1(1)$, such that $P = \sum_{i = 1}^d \theta_i^* P_i$ where based models $P_1,\ldots,P_d$ are given. 
The goal is to learn $\vtheta^*$ and minimize the cumulative expected regret $\E[\Reg^K]$, where
\begin{align}
    \Reg^K = \sum_{k = 1}^k [V^*(s_1) - V^{k}(s_1)]. \notag 
\end{align}

\section{Algorithm and Theory for Linear Bandits}
\label{sec:bandit}

\setlength{\textfloatsep}{0.1cm}
\setlength{\floatsep}{0.1cm}
\begin{algorithm}[t]
	\caption{VOFUL: \textbf{V}ariance-Aware \textbf{O}ptimism in the \textbf{F}ace of \textbf{U}ncertainty for \textbf{L}inear Bandits}
	\begin{algorithmic}[1] \label{algo:bandit}
		\STATE { \hspace{0ex}\textbf{Initialize:}  $\ell_{i} = 2^{2-i}, \iota = 16d \ln \frac{dK}{\delta}, L_2 = \lceil \log_2 K \rceil, \Lambda_2 = \{1,2,\ldots,L_2 + 1\}$, $\Theta_1 =\sB_2^d(1), $ Let $\gB$ be an $K^{-3}$-net of $\sB_2^d(2)$ with size not larger than $(\frac{4}{K})^{3d}$ \label{line:bandit_init}}
		\FOR{$k=1,2,\ldots,K$}
		\STATE{\textbf{Optimistic Action Selection:}} 
		\STATE{Observe context set $\gA_k \subseteq \sB_2^d(1)$}
		\STATE{Compute $\vx_k \gets \argmax_{\vx \in \gA_k}\max_{\vtheta \in \Theta_k} \vx \vtheta,$ choose action $\vx_k$}
		\STATE{Receive feedback $y_k$}
		\STATE{\textbf{Construct Confidence Set:}  
			\STATE For each $\vtheta \in \sB_2^d(1)$, define $\epsilon_k(\vtheta) = y_k - \vx_k \vtheta, \eta_k(\vtheta) = (\epsilon_k(\vtheta))^2$. \label{line:bandit_var_est}}
		\STATE{Define confidence set $\Theta_{k+1} = \bigcap_{j \in \Lambda_2}\Theta^{j}_{k+1},$ where
			\begin{align}
			\hspace{-2em}\Theta^{j}_{k+1} = \bigg\{ \vtheta \in \sB_2^d(1) ~ &:  \abs{\sum_{v = 1}^{k} \pj_j(\vx_{v} \vmu) \epsilon_{v}(\vtheta)}  \leq \sqrt{ \sum_{v = 1}^{k} \pj_j^2(\vx_{v} \vmu)\eta_{v}(\vtheta) \iota} + \ell_{j} \iota, \forall \vmu \in \gB \bigg\} \label{eq:bandit-confball}
			\end{align} 
			and $\pj_j(\cdot) = \clip(\cdot, \ell_j)$.
}
		\ENDFOR
	\end{algorithmic}
\end{algorithm} 

In this section, we introduce our algorithm for linear bandits and analyze its regret.
The pseudo-code is listed in Algorithm~\ref{algo:bandit}.
The following theorem shows our algorithm achieves the desired variance-dependent regret bound.
The full proof is deferred to Section~\ref{sec:proof_bandit}.
\begin{theorem}\label{thm:oful}
	The expected regret of Algorithm~\ref{algo:bandit} is bounded by 
	$
	\E[\Reg^K] \le \widetilde{O}( d^{4.5} \sqrt{\sum_{k=1}^K \sigma_k^2} + d^5)
	$.
\end{theorem}
This theorem shows our algorithm's regret has no explicit polynomial dependency on the number of rounds $K$.
In the worst-case where the variance is $\Omega\left(1\right)$, our bound becomes $\widetilde{O}\left(d^{4.5}\sqrt{K}+d^5\right)$, which has a worse dependency on $d$ compared with the minimax optimal algorithms~\citep{dani2008Stochastic,abbasi2011improved}.
However, in the benign case where the variance is $o(1)$, our bound can be much smaller.
In particular, in the noiseless case, our bound is a constant-type regret bound, up to logarithmic factors.
One future direction is to design an algorithm that is minimax optimal in the worst-case but also adapts to the variance magnitude like ours.

\subsection{Main Algorithm}
Now we describe our algorithm. Similar to many other linear bandit algorithms, the algorithm maintains  confidence sets $\{\Theta_k\}_{k\geq 1}$ for the underlying parameter $\vtheta^*$, and then choose the action greedily according to the confidence set.

To relax the known variance assumption, we use the following empirical Bernstein inequality that depends on the \emph{empirical variance}, in contrast to the Bernstein inequality that depends on the \emph{true variance}, which was used in existing works~\citep{zhou2020provably,faury2020improved}.

\begin{theorem} \label{thm:m0m1} Let  $\{\gF_i\}_{i = 0}^n$ be a filtration. Let $\{X_i\}_{i = 1}^n$ be a sequence of real-valued random variables such that $X_i$ is $\gF_{i}$-measurable. We assume that $\E[X_i \mid \gF_{i-1}]=0$ and that $\abs{X_i} \le b$ almost surely.  For $\delta < e^{-1},$ we have 
	\begin{align}
	\Pr[\abs{\sum_{i = 1}^n X_i} \le 8  \sqrt{\sum_{i=1}^n X_i^2 \ln\frac1\delta} + 16b\ln\frac1\delta] \ge 1- 6 \delta \log_2 n. \label{eq:m0m1}
	\end{align}
\end{theorem}
Importantly, this inequality controls the deviation via the empirical variance, which is $X_i^2$ and can be computed once $X_i$ is known. 
Note some previously proved inequalities require certain independence assumptions and thus cannot be directly applied to martingales \citep{maurer2009empirical,peel2013empirical}, so they cannot be used for solving our linear bandits problem. The proof of the theorem is deferred to Appendix~\ref{app:thm-m0m1}.

More effort is devoted to designing a confidence set that fully exploits the variance information. Note Theorem~\ref{thm:m0m1} is for real-valued random variables, and it remains unclear how to generalize it to the linear regression setting, which is crucial for building confidence sets for linear bandits. Previous works built up their confidence sets based on analyzing the ordinary ridged least square estimator \citep{dani2008Stochastic,abbasi2011improved}, or the weighted one \citep{zhou2020nearly}.

We drop the least square estimators and instead, we take a testing-based approach, as done in Equation~\eqref{eq:bandit-confball}.
To illustrate the idea, we first ignore the $\pj_j(\cdot)$ operation and $\ell_j$ terms. 
We define the noise function $\epsilon_k(\vtheta)$ and the variance function $\eta_k(\vtheta)$ (Line~\ref{line:bandit_var_est} of Algorithm~\ref{algo:bandit}). Note that $\epsilon_k(\vtheta^*) = \varepsilon_k$ and $\eta_k(\vtheta^*) = \varepsilon_k^2$, so we have the following fact: if $\vtheta = \vtheta^*,$ then Equation \eqref{eq:m0m1} would be true if we replace $X_k = w_k(\vmu)\epsilon_k(\vtheta)$ and $X_k^2 = w_k^2(\vmu)\eta_k(\vtheta)$ with high probability, where $\{w_k(\vmu)\}$ is a proper sequence of weights depending on the test direction $\vmu$.
Our approach uses the fact in the opposite direction: if weighted $w_k(\vmu)\epsilon_k(\vtheta), w_k^2(\vmu)\eta_k(\vtheta)$ satisfies Equation~\eqref{eq:m0m1} for all possible test directions $\vmu$ in an $K^{-3}$-net of the $d$-dimensional unit ball, then we put $\vtheta$ into the confidence set. 
\begin{remark}\label{remark0}
One can also view the algorithm as an elimination-based algorithm: if there exists some test direction $\vmu$ such that Equation~\eqref{eq:m0m1} fails for $X_k=w_k(\vmu)\epsilon_k(\vtheta)$ and $X_k^2= w_k^2(\vmu)\eta_k(\vtheta)$, then we eliminate $\vtheta$ from the confidence set permanently.
\end{remark}
Given the test direction $\vmu$, following the least square estimation, $w_k(\vmu) $ is set to be $\vx_{k}\vmu$. However, with $w_k(\vmu)=  \vx_k \vmu$, the right-hand-side of Equation~\eqref{eq:m0m1} is at least $b\geq \max_{1\leq k\leq n}|w_k(\vmu)|= \max_{1\leq k \leq n} |\vx_k \vmu|$, which might be dominant compared with $\sum_{k=1}^n w_k^{2}(\vmu)\eta_k(\theta)$ (See Appendix~\ref{app:diff} for 
a toy example).
To address this problem, we consider to peel $w_k(\vmu)$ for various thresholds of difference level. More precisely, we construct confidence regions respectively with $w_k^{j}(\vmu) =\clip_{j}(\vx_k\vmu)$, where $l_j = 2^{2-j}$ for $j=1,2,\ldots,\left\lceil \log_2K \right\rceil$. At last, we define the final confidence region as the intersections of all these confidence regions.

\begin{remark}\label{remark1}
Note that existing confidence sets in Equation~\eqref{eq:bandit-confball} either do not exploit variance information \citep{dani2008Stochastic,abbasi2011improved}, or require the variance to be known and do not fully exploit the variance information~\citep{zhou2020nearly,faury2020improved} as their regret bounds still have an $\widetilde{O}(\sqrt{K})$ term.
\end{remark}

\subsection{Proof Sketch of Theorem~\ref{thm:oful}}\label{sec:pfs}
Now we explain how our confidence set enables us to obtain a variance-dependent regret bound. We define $\vtheta_k = \argmax_{\vtheta \in \Theta_k} \vx_k(\vtheta - \vtheta^*)$ and $\vmu_k = \vtheta_k - \vtheta^*$. Then our goal is to bound the regret $\sum_k \vx_k \vmu_k.$ Our main idea is to consider $\{\vx_k\}, \{\vmu_k\}$ as two sequences of vectors. We decouple the complicated dependency between $\{\vx_k\}$ and $\{\vmu_k\}$ by a union bound over the net $\gB$ (defined in  Line~\ref{line:bandit_init} of Algorithm~\ref{algo:bandit}). To bound the regret, we implicitly divide all rounds $k \in [K]$ into norm layers based on $\log_2 \abs{\vx_k \vmu_k}$ in the analysis. \footnote{This cannot be done explicitly in the algorithm, since it would re-couple the two sequences.}
Within each layer, we apply Equation~\eqref{eq:bandit-confball} to obtain the relations between $\vmu_k$ and $\{\vx_1, \ldots, \vx_{k-1}\},$ 
which would self-normalize the growth of the two sequences, ensuring that their in-layer total sum is properly bounded. Since we have logarithmically many layers, the total regret is then properly bounded. We highlight that our norm peeling technique ensures that the variance-dependent term dominates the other variance-independent term in Bernstein inequalities ($\sqrt{\sum X_i^2} ~\substack{\succ\\ \sim}~ b$ in Theorem~\ref{thm:m0m1}), which resolves the variance-independent term in the final regret bound obtained by \citet{zhou2020nearly}. 

We start the analysis by proving that the probability of failure events (i.e., the events where $\theta^*\notin\Theta_k$ for some $k\in [K]$) 
is properly bounded (see Lemma~\ref{lem:bandit-hpe}).
Assuming the successful events happen, we have that $\theta^*\in \Theta_k$ for all $k\in [K]$. Then we obtain that.
\begin{align}
    \Reg^K& :=\sum_{k=1}^K\left( \max_{\vx\in \mathcal{A}_k}\vtheta^* - \vx_k \vtheta^*\right) \leq \sum_{k=1}^K \max_{\vx\in \mathcal{A}_k,\vtheta\in \Theta_k}\vx_k(\vtheta_k-\vtheta^*)\leq \sum_{k=1}^K \vx_k (\vtheta_k - \vtheta^*)=\sum_{k=1}^K \vx_k \vmu_k.\nonumber
\end{align}

Next we divide the time steps $[K]$ into $L_2+1$ disjoint subsets $\{\mathcal{K}_j\}_{j=1}^{L_2+1}$ according to the magnitude of $\vx_k\vmu_k$. More precisely, for for $1\leq j \leq L_2$  we assign $k$ to $\mathcal{K}_j$ iff $\vx_k\vmu_k \in (l_j/2,l_j]$, and for $j = L_2+1$, we assign $k$ to $\mathcal{K}_j$ iff $\vx_k\vmu_k\leq l_{l_2+1}/2$. Define
\begin{align}
  \Phi_k^{j}(\vmu) = \sum_{v=1}^{k-1} \pj_j(\vx_{v}\vmu) \vx_{v}\vmu + \ell_j^2, \qquad 
   \Psi_k^{j}(\vmu) = \sum_{v=1}^{k-1}  \pj_j^2(\vx_{v} \vmu) \eta_v(\vtheta^*). \label{eq:bandit-phi-psi}
\end{align}
By the definition of $\Theta_k$ in \eqref{eq:bandit-confball}, we have that (see Claim~\ref{claim:bandit_multiply_1})
\begin{align}
    \sum_{k=1}^K\vx_k \vmu_k \leq 1 +\sum_{j=1}^{L_2} \sum_{k\in \mathcal{K}_j} \vx_k \vmu_k \times  \frac{3\sqrt{\Psi_k^{j}(\vmu_k) \iota} + \sqrt{\sum_{v = 1}^{k-1} 2\pj^2_{j}(\vx_{v} \vmu_k) (\vx_v \vmu_k)^2 \iota} + 3  \ell_{j}\iota }{\Phi_k^{j}(\vmu_k)}.\label{eq:cr1}
\end{align}

Continuing the computation, we have that
\begin{align} \hspace{-2em}
 &\quad \sum_{j=1}^{L_2} \sum_{k\in \mathcal{K}_j} \vx_k \vmu_k \frac{ \sqrt{\sum_{v = 1}^{k-1} 2\pj^2_{j}(\vx_{v} \vmu_k) (\vx_v \vmu_k)^2\iota}}{\Phi_k^{j}(\vmu_k)} \notag  \\
 &\le \frac{1}{2} \sum_{j=1}^{L_2} \sum_{k\in \mathcal{K}_j} \vx_k \vmu_k +  \sum_{j=1}^{L_2} \sum_{k\in \mathcal{K}_j}  \vx_k \vmu_k  \sI\bigg\{\frac{\sqrt{\sum_{v = 1}^{k-1} 2\pj^2_{j}(\vx_{v} \vmu_k) (\vx_v \vmu_k)^2\iota }}{\Phi_k^{j}(\vmu_k)} > \frac{1}{2}\bigg\}\nonumber
  \\ &    \le \frac{1}{2} \sum_{j=1}^{L_2} \sum_{k\in \mathcal{K}_j} \vx_k \vmu_k +  \sum_{j=1}^{L_2} \sum_{k\in \mathcal{K}_j}  \vx_k \vmu_k  \frac{4l_j\iota}{\Phi_k^j(\vmu_k)} \label{eq:cr1.5}
    \\ &    \le \frac{1}{2} \sum_{j=1}^{L_2} \sum_{k\in \mathcal{K}_j} \vx_k \vmu_k + O( d^4 \abs{\Lambda_2}\iota \log^3(dK)),\label{eq:cr2}
\end{align}
where \eqref{eq:cr1.5} is by the fact that  $\frac{\sqrt{\sum_{v = 1}^{k-1} 2\pj^2_{j}(\vx_{v} \vmu_k) (\vx_v \vmu_k)^2\iota }}{\Phi_k^{j}(\vmu_k)} > \frac{1}{2}$ implies that $\frac{4l_j\iota}{\Phi_k^j(\vmu_k)}>1$, 
and \eqref{eq:cr2} follows by Lemma~\ref{lemma:bound_m2}. By \eqref{eq:cr1} and \eqref{eq:cr2}, we have that
\begin{align}
    &\sum_{k=1}^K\vx_k\vmu_k  \leq   12\sum_{j=1}^{L_2} \sum_{k\in \mathcal{K}_j} \vx_k \vmu_k \times \frac{\sqrt{\Psi_k^{j}(\vmu_k)\iota}}{\Phi^{j}_k(\vmu_k)}+\tilde{O}(d^5) \nonumber
\\&     \leq  \sum_{j=1}^{L_2}\sum_{k\in \mathcal{K}_j} \frac{12 \vx_k \vmu_k \ell_{j}}{{\Phi^{j}_k(\vmu_k)}} \sqrt{\sum_{k=1}^{K} \eta_k(\vtheta^*)\iota}+\tilde{O}(d^5)   \le  O(d^4 \abs{\Lambda_2} \log^3(dK))   \sqrt{\Big(\ln\frac1\delta + \sum_{k=1}^K \sigma_k^2\Big) \iota} +\tilde{O}(d^5),\nonumber
\end{align}
where the last inequality uses Lemma~\ref{lemma:bound_m2}. Therefore, the regret bound is $\tilde{O}\left(d^{4.5}\sqrt{\sum_{k=1}^K \sigma_k^2} +d^5\right)$.
See Section~\ref{sec:proof_bandit} for the full proof.

\section{Algorithm and Theory for Linear Mixture MDP}
\label{sec:rl}

We introduce our algorithm and the regret bound for linear mixture MDP.
Its pseudo-code is listed in Algorithm~\ref{alg:main} and its regret bound is stated below. The proof is deferred to Section~\ref{sec:proof_rl}.
\begin{theorem}\label{thm:main_rl}
	The expected regret of Algorithm~\ref{alg:main} is bounded by
	$\E[\Reg^K ] \leq  \widetilde{O}\left(d^{4.5}\sqrt{ K }+d^9\right)$.
\end{theorem}

\setlength{\textfloatsep}{0.1cm}
\setlength{\floatsep}{0.1cm}
\begin{algorithm}[t]
	\caption{VARLin: \textbf{V}ariance-\textbf{A}ware \textbf{R}L with \textbf{Lin}ear Function Approximation}
	\begin{algorithmic}[1]\label{alg:main}
		\STATE { \hspace{0ex}\textbf{Initialize:}  $\ell_{i} = 2^{2-i}, \iota = 16d\ln\frac{dHK}{\delta}, L_0 = \lceil \log_2 KH \rceil, L_1 = L_2 = \lceil 5 \log_2(HK) + 3 \rceil, \Lambda_0 = \{0,1,,\ldots,L_0\}, \Lambda_1 = \{1,\ldots,L_1\}, \Lambda_2 = \{1,\ldots,L_2\}$.  $\gB$ be an $(HK)^{-3}$-net of $\sB_1^d(2)$ with size no larger than $(\frac{4}{HK})^{3d}$.  $\Theta_1 =\sB_1^d(1).$ \label{line:rl_init}
		}
		\FOR{$k=1,2,\ldots,K$}
		\STATE{\textbf{Optimistic Planning}:}
		\FOR{$h=H,H-1,\ldots,1$} \label{loc:line1}
		\STATE{For each $(s, a) \in \gS \times \gA,$ let $Q_h^k(s,a) = \min\{1, r(s,a)+ \max_{\vtheta\in \Theta_{k}} \sum_{i = 1}^d \theta_i P^i_{s,a}V_{h+1}^k \}$.}
		\STATE{For each $s \in \gS,$ let $V_h^k(s) = \max_{a \in \gA}Q_h^k(s,a)$.}	
		\ENDFOR
		\FOR{$h=1,2,\ldots, H$}
		\STATE{Choose action $a_h^k \gets \argmax_{a \in \gA} Q_h^k(s_h^k, a)$, observe the next state $s_{h+1}^k$.}
		\ENDFOR
		
		\STATE{\textbf{Construct Confidence Set}:}
		\STATE{For $m\in \Lambda_0, h\in[H],$ define the input  $\vx_{k,h}^m = [P^1_{s_h^k, a_h^k} (V_{h+1}^k)^{2^m}, \ldots, P^d_{s_h^k, a_h^k}(V_{h + 1}^k)^{2^m}]^\top$.}
		\STATE{For $m\in \Lambda_0, h\in[H],$ define the variance estimate $\eta_{k,h}^m = \max_{\vtheta \in \Theta_k} \{\vtheta \vx_{k,h}^{m+1} - (\vtheta \vx_{k,h}^m)^2\}$.
		}
		\STATE{Denote $\epsilon_{v,u}^{m}(\vtheta) = \vtheta \vx_{v,u}^m-(V_{u+1}^{v}(s_{u+1}^{v}))^{2^m}$ for $m\in \Lambda_0, u \in [H],v \in [k-1]$ }
		\STATE{Define $\gT^{m,i}_{k+1} = \{(v, u) \in [k]\times [H]: \eta^m_{v,u}\in (\ell_{i+1}, \ell_i]\},\gT^{m, L_1 + 1}_{k+1} = \{(v, u) \in [k] \times [H]: \eta^m_{v,u} \le \ell_{L_1 + 1}\}$.
		}
		\STATE{Define the confidence ball $\Theta_{k+1} = \bigcap_{m,i,j} \Theta_{k+1}^{m,i,j},$ where 
			\begin{align}
			\Theta_{k+1}^{m,i,j} = \Bigg\{ \vtheta \in \sB_1^d(1) :  & \abs{\sum_{(v, u) \in \gT_{k}^{m,i}} \pj_j(\vx_{v, u}^m \vmu) \epsilon_{v,u}^m(\vtheta)} \notag \\ \leq &4\sqrt{ \sum_{(v, u) \in \gT_{k}^{m,i}} \pj_j^2(\vx_{v,u}^{m} \vmu)\eta_{v,u}^m \iota} +4\ell_{j}\iota, \forall \vmu \in \gB \Bigg\} \label{eq:rl-conf-ball}
			\end{align}
			and $\pj_j(\cdot) = \clip(\cdot, \ell_j)$
			\label{line:indices}
		}
		\ENDFOR
	\end{algorithmic}
\end{algorithm}

Before describing our algorithm, we introduce some additional notations.
In this section, we assume that, unless explicitly stated, the variables $m, i, j, k, h$ iterate over the sets $\Lambda_0, \Lambda_1, \Lambda_2, [K], [H],$ respectively.
See Line~\ref{line:rl_init} of Algorithm~\ref{alg:main} for the definitions of these sets.
For example, at Line~\ref{line:indices} of Algorithm~\ref{alg:main}, we have $\bigcap_{m, i, j} \Theta_{k+1}^{m,i,j} = \bigcap_{m \in \Lambda_0, i \in \Lambda_1, j \in \Lambda_2} \Theta_{k+1}^{m,i,j}.$ 

The starting point of our algorithm design is from \citet{zhang2020reinforcement}, in which the authors obtained a nearly horizon-free regret bound in tabular MDP. 
A natural idea is to combine their proof with our results for linear bandits  and obtain a nearly horizon-free regret bound for linear mixture MDP.

Note that, however, there is one caveat for such direct combination: in Section~\ref{sec:bandit}, the confidence set $\Theta_k$ is updated at a per-round level, in that $\Theta_k$ is built using all rounds prior to $k$; while for the RL setting, the confidence set $\Theta_k$ could only be updated at a per-episode level and use all time steps prior to episode $k.$ Were it updated at a per-time-step level, severe dependency issues would prevent us from bounding the regret properly. Such discrepancy in update frequency results in a gap between the confidence set built using data prior to episode $k,$ and that built using data prior to time step $(k,h).$ 
 Fortunately, we are able to resolve this issue. In  Lemma~\ref{lem:boundI}, we show that we can relate these two confidence intervals, except for $\tilde{O}(d)$ ``bad'' episodes. Therefore, we could adapt the analysis in \citet{zhang2020reinforcement} only for the not ``bad'' episodes, and we bound the regret by $1$ for the ``bad'' episodes. The resulting regret bound should be $\widetilde{O}(d^{6.5}\sqrt{K}).$

To further reduce the horizon-free regret bound to $\tilde{O}(d^{4.5}\sqrt{K})$, we present another novel technique. We first note an important advantage of the linear mixture MDP setting over the linear bandit setting: in the latter setting, we cannot estimate the variance because there is no structure on the variance among different actions; while in the former setting, we could estimate an upper bound of the variance, because the variance is a quadratic function of $\vtheta^*$.
Therefore, we can use the peeling technique on the \emph{variance magnitude} to reduce the regret (comparing Equation~\eqref{eq:bandit-6200} and Equation~\eqref{lem5-1000} in appendix).
We note that one can also apply this step to linear bandits if the variance can be estimated.

Along the way, we also need to bound the gap between estimated variance and true variance, which can be seen as the ``regret of variance predictions.''
Using the same idea, we can build a confidence set using  the variance sequence ($\vx^2$), and the regret of variance predictions can be bounded by the variance of variance, namely the 4-th moment. 
Still, a peeling step on the 4-th moment is required to bound the regret of variance predictions, we need to bound the gap between estimated 4-th moment and true 4-th moment, which requires predicting 8-th moment, 
We continue to use this idea: we estimate 2-th, 4-th, 8-th, \ldots, $O(\log KH)$-th moments.
The index $m$ is used for moments, and $\Lambda_0$ is the index set reserved for moments.
We note that the proof in \citep{zhang2020reinforcement} also depends on the higher moments.
The main difference is here we estimate these higher moments explicitly.

\section{Discussions}
\label{sec:conclusion}

By incorporating the variance information in the confidence set construction, we derive the first variance-dependent regret bound for linear bandits and the nearly horizon-free regret bound for linear mixture MDP.
Below we discuss limitations of our work and some future directions.

One drawback of our result is that our dependency on $d$ is large.
The main reason is our bounds rely on the convex potential lemma (Lemma~\ref{lemma:bound_m2}), which is $\widetilde{O}(d^4)$.
In analogous to the elliptical potential lemma in \citep{abbasi2011improved}, we believe that this bound can be improved  to $\widetilde{O}(d).$
This improvement will directly reduce the dependencies on $d$ in our bounds.

Another drawback is that our method is not computationally efficient.
This is a common issue in elimination-based algorithms.
We note that the issue of computational tractability is common in sequential decision-making problems~\citep{zhang2019regret,wang2020long,bartlett2012regal,zanette2020learning,krishnamurthy2016pac,jiang2017contextual,sun2018model,jin2021bellman,du2021bilinear,dong2020root}.
We leave it as a future direction to design computationally efficient algorithms that enjoy variance-dependent bounds for settinsg considered in this paper.

Lastly, in this paper, we only study linear function approximation. It would be interesting to generalize the ideas in this paper to other settings with function approximation schemes~\citep{yang2019sample,jin2019provably,zanette2020learning,wang2020provably,russo2013eluder,jiang2017contextual,sun2018model,du2021bilinear,jin2021bellman}.

\section*{Acknowledgement}
Simon S. Du gratefully acknowledges funding from NSF Award’s IIS-2110170 and DMS-2134106.

\bibliography{references}
\bibliographystyle{plainnat}

\section*{Checklist}


\begin{enumerate}

\item For all authors...
\begin{enumerate}
  \item Do the main claims made in the abstract and introduction accurately reflect the paper's contributions and scope?
    \answerYes
  \item Did you describe the limitations of your work?
    \answerYes{We discuss the limitations in Section~\ref{sec:conclusion}.}
  \item Did you discuss any potential negative societal impacts of your work?
    \answerNA{This work is theoretical so the boarder impact does not apply.}
  \item Have you read the ethics review guidelines and ensured that your paper conforms to them?
    \answerYes
\end{enumerate}

\item If you are including theoretical results...
\begin{enumerate}
  \item Did you state the full set of assumptions of all theoretical results?
    \answerYes{We present the main assumption in Section~\ref{sec:pre}.}
	\item Did you include complete proofs of all theoretical results?
    \answerYes{We present the proofs in Appendix.}
\end{enumerate}

\item If you ran experiments...
\begin{enumerate}
  \item Did you include the code, data, and instructions needed to reproduce the main experimental results (either in the supplemental material or as a URL)?
    \answerNA{We have no experiments.}
  \item Did you specify all the training details (e.g., data splits, hyperparameters, how they were chosen)?
    \answerNA
	\item Did you report error bars (e.g., with respect to the random seed after running experiments multiple times)?
    \answerNA
	\item Did you include the total amount of compute and the type of resources used (e.g., type of GPUs, internal cluster, or cloud provider)?
    \answerNA
\end{enumerate}

\item If you are using existing assets (e.g., code, data, models) or curating/releasing new assets...
\begin{enumerate}
  \item If your work uses existing assets, did you cite the creators?
    \answerNA{We do not use existing models.}
  \item Did you mention the license of the assets?
    \answerNA
  \item Did you include any new assets either in the supplemental material or as a URL?
    \answerNA
  \item Did you discuss whether and how consent was obtained from people whose data you're using/curating?
    \answerNA
  \item Did you discuss whether the data you are using/curating contains personally identifiable information or offensive content?
    \answerNA
\end{enumerate}

\item If you used crowdsourcing or conducted research with human subjects...
\begin{enumerate}
  \item Did you include the full text of instructions given to participants and screenshots, if applicable?
    \answerNA{This work is inrelevent with human subjects.}
  \item Did you describe any potential participant risks, with links to Institutional Review Board (IRB) approvals, if applicable?
    \answerNA
  \item Did you include the estimated hourly wage paid to participants and the total amount spent on participant compensation?
    \answerNA
\end{enumerate}

\end{enumerate}

\newpage
\appendix
\section{Technical Lemmas}
\begin{lemma}[\citep{azuma1967weighted}] \label{lem:azuma} Let $(M_n)_{n \ge 0}$ be a martingale such that $M_0 = 0$ and $\abs{M_n - M_{n - 1}} \le b$ almost surely for every $n \ge 1$. Then we have 
\begin{align*}
    \Pr[\abs{M_n} \ge b\sqrt{2n \log(2/\delta)}] \le \delta.
\end{align*}

\end{lemma}

\begin{lemma}[\citep{zhang2020model}, Lemma 9] \label{lem:nine2} Let $\{\gF_i\}_{i\ge 0}$ be a filtration. Let $\{X_i\}_{i\ge 1}$ be a real-valued stochastic process adapted to $\{\gF_i\}_{i\ge 0}$ such that $0 \le X_i \le 1$ almost surely and that $X_i$ is $\gF_i$-measurable. For every $\delta \in (0, 1), c \ge 1,$ we have
\begin{align*}
    \Pr[\exists n \ge 1 : \sum_{i = 1}^n \E[X_i \mid \gF_{i-1}]  \ge 4 c \ln\frac4\delta, \sum_{i=1}^n X_i  \le c \ln\frac4\delta] \le \delta.
\end{align*}
\end{lemma}
\begin{lemma} \label{lem:nine} Let $\{\gF_i\}_{i\ge 0}$ be a filtration. Let $\{X_i\}_{i\ge 1}$ be a real-valued stochastic process adapted to $\{\gF_i\}_{i\ge 0}$ such that $0 \le X_i \le 1$ almost surely and that $X_i$ is $\gF_i$-measurable. For every $\delta \in (0, 1), c \ge 1,$ we have
\begin{align*}
    \Pr[\exists n \ge 1 : \sum_{i = 1}^n X_i \ge 4 c \ln\frac4\delta, \sum_{i=1}^n  \E[X_i \mid \gF_{i-1}] \le c \ln\frac4\delta] \le \delta.
\end{align*}
\end{lemma}

\begin{proof} We follow the proof of Lemma 9 in \citep{zhang2020model}. Let $\lambda > 0$ be a parameter, $\mu_i = \E[X_i \mid \gF_{i-1}].$ Define $Y_n = \exp(\lambda \sum_{i = 1}^n X_i - (e^{\lambda}-1)\sum_{i = 1}^n \mu_i)$ for $n \ge 0.$ Note that $\E[e^{\lambda X}] \le \mu e^\lambda + (1-\mu) \le e^{\mu(e^\lambda - 1)},$ so $\E[e^{\lambda X_i - (e^\lambda - 1) \mu_i} \mid \gF_{i-1}] \le 1,$ thus $\{Y_n\}_{n \ge 0}$ is a super-martingale. Let $\tau = \min\{n : \sum_{i = 1}^n X_i \ge 4 c \ln(4/\delta)\}$ be a stopping time, then we have $\abs{Y_{\min\{\tau, n\}}} \le e^{\lambda (4c\ln(4/\delta) + 1)} < +\infty$ almost surely for every $n \ge 0.$ Therefore, by the optional stopping theorem, we have $\E[Y_\tau] \le 1.$ Finally, we have 
\begin{align*}
\Pr[\exists n \ge 1 : \sum_{i = 1}^n X_i \ge 4 c \ln\frac4\delta, \sum_{i=1}^n \mu_i \le c \ln\frac4\delta] &\le \Pr[\sum_{i = 1}^\tau \mu_i \le c \ln\frac4\delta] \\
&\le \Pr[Y_\tau \ge \exp(\lambda \sum_{i = 1}^\tau X_i - (e^{\lambda}-1) c \ln\frac4\delta)] \\
&\le  \Pr[Y_\tau \ge \exp(\lambda (4c\ln\frac4\delta -1) - (e^{\lambda}-1) c \ln\frac4\delta)]  \\ 
&\le \exp(\lambda  (1-4c\ln\frac2\delta) + (e^{\lambda}-1) c \ln\frac4\delta)\\
&= e^\lambda e^{(e^\lambda - 1 - 4\lambda)c \ln(4/\delta)}.
\end{align*}
Choosing $\lambda = 1,$ we have 
\begin{align*}
    e^\lambda e^{(e^\lambda - 1 - 4\lambda)c \ln(4/\delta)} \le e \cdot e^{-2c\ln(4/\delta)} = e (\frac{\delta}{4})^c \le \frac{e}{4} \delta \le \delta,
\end{align*}
which concludes the proof.
\end{proof}

\begin{lemma} \label{lem:mchernoff2} Let $\{\gF_i\}_{i\ge 0}$ be a filtration. Let $\{X_i\}_{i = 1}^n$ be a sequence of random variables such that $\abs{X_i} \le 1$ almost surely, that $X_i$ is $\gF_i$-measurable. For every $\delta \in (0, 1),$ we have
\begin{align*}
    \Pr[\sum_{i = 1}^n \E[X_i^2\mid \gF_{i-1}] \ge  \sum_{i = 1}^n 8 X_i^2 + 4 \ln\frac4\delta] \le (\lceil \log_2 n \rceil + 1) \delta.
\end{align*}
\end{lemma}

\begin{proof} Let $Y = \sum_{i = 1}^n \E[X_i^2 \mid \gF_{i-1}], Z = \sum_{i = 1}^n X_i^2.$ Applying Lemma~\ref{lem:nine2} with the sequence $\{X_i^2\}_{i=1}^n,$ we have for every $c \ge 1,$
\begin{align*}
    \Pr[Y \ge 4 c \ln\frac4\delta, Z \le c \ln\frac4\delta] \le \delta.
\end{align*}
Therefore, we have 
\begin{align*}
&\quad \Pr[Y\ge 8 Z + 4 \ln\frac4\delta]  \\
&\le \sum_{j = 1}^{\lceil \log_2 n \rceil} \Pr[Y \ge 8 Z + 4 \ln\frac4\delta, 2^{j-1} \ln\frac4\delta \le Z \le 2^j \ln\frac4\delta] + \Pr[Y \ge 8Z + 4 \ln\frac4\delta, Z \le  \ln\frac4\delta] \\
&\le \sum_{j = 1}^{\lceil \log_2 n \rceil} \Pr[Y \ge 8 Z, 2^{j-1} \ln\frac4\delta \le Z \le 2^j \ln\frac4\delta] + \Pr[Y \ge 4 \ln\frac4\delta, Z \le  \ln\frac4\delta] \\
&\le \sum_{j = 1}^{\lceil \log_2 n \rceil} \Pr[Y \ge 8 \cdot 2^{j-1} \ln\frac4\delta, 2^{j-1} \ln\frac4\delta \le Z \le 2^j \ln\frac4\delta] + \Pr[Y \ge 4 \ln\frac4\delta, Z \le  \ln\frac4\delta] \\
&\le \sum_{j = 1}^{\lceil \log_2 n \rceil} \Pr[Y \ge 4  \cdot 2^{j} \ln\frac4\delta,  Z \le 2^j \ln\frac4\delta] + \Pr[Y \ge 4 \ln\frac4\delta, Z \le  \ln\frac4\delta] \\
&\le (\lceil \log_2 n \rceil + 1) \delta
\end{align*}
as desired.
\end{proof}

\begin{lemma} \label{lem:mchernoff} Let $\{\gF_i\}_{i\ge 0}$ be a filtration. Let $\{X_i\}_{i = 1}^n$ be a sequence of random variables such that $\abs{X_i} \le 1$ almost surely, that $X_i$ is $\gF_i$-measurable. For every $\delta \in (0, 1),$ we have
\begin{align*}
    \Pr[\sum_{i = 1}^n X_i^2 \ge  \sum_{i = 1}^n 8\E[X_i^2\mid \gF_{i-1}] + 4 \ln\frac4\delta] \le (\lceil \log_2 n \rceil + 1) \delta.
\end{align*}
\end{lemma}

\begin{proof} Let $Y = \sum_{i = 1}^n X_i^2, Z = \sum_{i = 1}^n \E[X_i^2 \mid \gF_{i-1}].$ Applying Lemma~\ref{lem:nine} with the sequence $\{X_i^2\}_{i=1}^n,$ we have for every $c \ge 1,$
\begin{align*}
    \Pr[Y \ge 4 c \ln\frac4\delta, Z \le c \ln\frac4\delta] \le \delta.
\end{align*}
Therefore, we have 
\begin{align*}
&\quad \Pr[Y\ge 8 Z + 4 \ln\frac4\delta]  \\
&\le \sum_{j = 1}^{\lceil \log_2 n \rceil} \Pr[Y \ge 8 Z + 4 \ln\frac4\delta, 2^{j-1} \ln\frac4\delta \le Z \le 2^j \ln\frac4\delta] + \Pr[Y \ge 8Z + 4 \ln\frac4\delta, Z \le  \ln\frac4\delta] \\
&\le \sum_{j = 1}^{\lceil \log_2 n \rceil} \Pr[Y \ge 8 Z, 2^{j-1} \ln\frac4\delta \le Z \le 2^j \ln\frac4\delta] + \Pr[Y \ge 4 \ln\frac4\delta, Z \le  \ln\frac4\delta] \\
&\le \sum_{j = 1}^{\lceil \log_2 n \rceil} \Pr[Y \ge 8 \cdot 2^{j-1} \ln\frac4\delta, 2^{j-1} \ln\frac4\delta \le Z \le 2^j \ln\frac4\delta] + \Pr[Y \ge 4 \ln\frac4\delta, Z \le  \ln\frac4\delta] \\
&\le \sum_{j = 1}^{\lceil \log_2 n \rceil} \Pr[Y \ge 4  \cdot 2^{j} \ln\frac4\delta,  Z \le 2^j \ln\frac4\delta] + \Pr[Y \ge 4 \ln\frac4\delta, Z \le  \ln\frac4\delta] \\
&\le (\lceil \log_2 n \rceil + 1) \delta
\end{align*}
as desired.
\end{proof}

\begin{lemma}[\citep{zhang2020model}, Lemma 11] \label{lem:ten} Let $(M_n)_{n \ge 0}$ be a martingale such that $M_0 = 0$ and $\abs{M_n - M_{n - 1}} \le b$ almost surely for every $n \ge 1$. For each $n \ge 0$, let $\gF_n = \sigma(M_0, \ldots, M_n)$ and let $\Var_n = \sum_{i = 1}^n \E[(M_i - M_{i - 1})^2 \mid \gF_{i - 1}]$. Then for any $n \ge 1$ and $ \epsilon, \delta > 0$, we have
\begin{align}
    \Pr[\abs{M_n} \ge 2 \sqrt{2 \Var_n \ln(1/\delta)} + 2 \sqrt{\epsilon \ln(1/\delta)} + 2 b \ln(1/\delta)] \le 2(\log_2(b^2 n/\epsilon) + 1)\delta.  \notag
\end{align}
\end{lemma}

\begin{lemma}\label{lemma:sequence2}
Let $\lambda_{1},\lambda_2, \lambda_4>0$, $\lambda_3\geq 1$ and $\kappa = \max\{\log_{2}(\lambda_1),1 \}$. Let $a_1,a_2, \ldots ,a_{\kappa}$ be non-negative reals such that $a_{i}\leq \lambda_1$ and $a_{i}\leq \lambda_{2}\sqrt{a_{i}+a_{i+1} + 2^{i + 1} \lambda_3}+\lambda_4$ for any $1\leq i\leq \kappa$ (with $a_{\kappa+1}=\lambda_{1}$). Then we have that
\begin{align}
    a_{1}\leq 22 \lambda_2^2 + 6 \lambda_4 + 4\lambda_2 \sqrt{2\lambda_3}.\nonumber
\end{align}
\end{lemma}

\begin{proof} Note that 
\begin{align*}
    a_i \le \lambda_2 \sqrt{a_i} + \lambda_2 \sqrt{a_{i + 1} + 2^{i + 1} \lambda_3} + \lambda_4, 
\end{align*}
so we have 
\begin{align*}
    a_i \le \left(\lambda_2 + \sqrt{\lambda_2 \sqrt{a_{i + 1} + 2^{i + 1} \lambda_3} + \lambda_4} \right)^2 \le 2\lambda_2^2 + 2\lambda_2 \sqrt{a_{i + 1} + 2^{i + 1} \lambda_3} + 2 \lambda_4.
\end{align*}
By Lemma 11 in \citep{zhang2020reinforcement}, we have 
\begin{align*}
    a_1 &\le \max\left\{\left(2 \lambda_2 + \sqrt{(2\lambda_2)^2 + (2 \lambda_2^2 + 2 \lambda_4)} \right)^2, 2\lambda_2 \sqrt{8 \lambda_3} + 2\lambda_2^2 + 2 \lambda_4 \right\} \\
    &\le \max\{20 \lambda_2^2 +4 \lambda_4, 2\lambda_2 \sqrt{8 \lambda_3} + 2\lambda_2^2 + 2 \lambda_4 \} \le 22 \lambda_2^2 + 6 \lambda_4 + 4\lambda_2 \sqrt{2\lambda_3},
\end{align*}
which concludes the proof.
\end{proof}

\section{Limitations of Previous Approaches}\label{app:diff}

 In the example in Section~\ref{sec:intro}, if we know $x_i \leq \sqrt{\frac{1}{K}}$ for $1\leq i \leq K$, the best confidence region for $\theta^*$ should be $\Theta_{t} = \{\theta| \|\theta - \hat{\theta}_t\|_{\Lambda_{t-1}} \leq C(\sigma\sqrt{d} +\lambda^{1/2})  \}$, and we can obtain a variance-aware regret bound by letting $\lambda = \sigma^{2}$. However, if we let $x_{K+1}= 1$ and use the same concentration inequality as before, the confidence region would be $\Theta_{K+1} = \{ \theta | \|\theta - \hat{\theta}_t \|_{\Lambda_{t-1}} \}\leq C(\sigma \sqrt{d}+1+\lambda^{1/2})$. 

We present the detailed computation as below. Choose $\theta^* = \Theta(1)$. $\theta^* - \hat{\theta}_{K+1} = -\frac{\sum_{i=1}^{K+1} x_i\epsilon_i}{\lambda +\sum_{i=1}^{K+1} x_i^2 }+\frac{\lambda\theta^*}{\lambda + \sum_{i=1}^{K+1}x_i^2}$. When $\epsilon_i$ is bounded in $[-1,1]$ with variance $\sigma^2$, following Bernstein inequality, we have that $\left|\frac{\sum_{i=1}^{K+1} x_i\epsilon_i}{\lambda +\sum_{i=1}^{K+1} x_i^2 }\right|\leq \frac{\sqrt{\sigma^2\sum_{i=1}^{K+1} x_i^2}+\max_i x_{i}}{\lambda+\sum_{i=1}^{K+1} x_i^2}$. Therefore, the best confidence interval we have is
\begin{align*}
    \|\theta^* - \hat{\theta}_{K+1} \|_{\Lambda_{K}} ~\substack{\prec\\ \sim}~ & \sqrt{\frac{\sigma^2 \sum_{i=1}^{K+1} x_i^2}{\lambda+\sum_{i=1}^{K+1} x_i^2}}+ \frac{\max_i x_i}{\sqrt{\lambda + \sum_{i=1}^{K+1} x_i^2}} +\frac{\lambda\theta^*}{ \sqrt{\lambda + \sum_{i=1}^{K+1}x_i^2}}  \\
    = &\Theta\left(\sqrt{\frac{\sigma^2}{\lambda +1}}+ \frac{1+\lambda}{\sqrt{1+\lambda}}  \right),
\end{align*}
i.e., $|\theta^* - \hat{\theta}_{K+1}|_{\Lambda_{K}}~\substack{\prec\\ \sim}~ \Theta(\sigma + \lambda^{1/2}+1)$. Therefore, to maintain a confidence region for the general case following methods in \citep{zhou2020nearly,faury2020improved}, the term $1+\lambda^{1/2}$ is unavoidable.

This counter example highlights the necessity of our peeling step in the algorithm.

\begin{remark}
We note that for $\sigma$-sub-Gaussian noise (instead of $\sigma^2$ variance and $1$-sub-Gaussian), one can ensure that $\left|\frac{\sum_{i=1}^{K+1} x_i\epsilon_i}{\lambda +\sum_{i=1}^{K+1} x_i^2 }\right|\leq \frac{\sqrt{\sigma^2\sum_{i=1}^{K+1} x_i^2}}{\lambda+\sum_{i=1}^{K+1} x_i^2}$, which help to reduce the width of confidence interval and obtain $|\theta^*-\hat{\theta}_{K+1}|_{\Lambda_K}\leq O(\sigma + \lambda^{1/2})$.
\end{remark}

\section{
Proof of Lemma~\ref{lemma:gpl}
}\label{sec:gpl}
In this section, we present the proof of Lemma~\ref{lemma:gpl}. 

\textbf{Restatement of Lemma~\ref{lemma:gpl}}\emph{
Let $f(x)\geq 0$ be a convex function over $\mathbb{R}$ such that $ \frac{f(x)}{x^2} \leq \frac{f(y)}{y^2}\leq 1$ and $f(x)\geq f(y)$ if $x^2\geq y^2>0$. Fix $\ell\in (0,1]$. For any $\vx_1,\vx_2,\ldots, \vx_{t}\in \mathbb{B}_2^{d}(1)$ and $\vmu_1,\vmu_2,\ldots,\vmu_t\in \mathbb{B}_2^d(1)$, we have that}
\begin{align}
    \sum_{i=1}^{t} \min \left\{\frac{f(\vx_i \vmu_i)}{ \sum_{j=1}^{i-1} f(\vx_j \vmu_i) +\ell^2},1 \right\} \leq O(d^4\log(Cdt/\ell)).\label{eq:gpl1}
\end{align}

Let $f(x)$ and $\ell$ be fixed.
To prove Lemma~\ref{lemma:gpl}, we have the lemmas below.
\begin{lemma}\label{lemma:gpl_ho}
For any $\vx_1.\vx_2,\ldots,\vx_{t}\in \mathbb{B}_2^d(1)$ and $\vmu_1,\vmu_2,\ldots,\vmu_n\in \mathbb{B}_2^d(1)$, we have that
\begin{align}
     \sum_{i=1}^{t} \min \left\{\frac{f(\vx_i \vmu_i)}{ \sum_{j=1}^{t} f(\vx_j \vmu_i) +\ell^2},1 \right\} \leq O(d\log(Cdt/\ell)).\label{eq:gpl_ho}
\end{align}
\end{lemma}

\begin{lemma}\label{lemma:ener_bound}
Let $\vx_{1},\vx_{2},\ldots,\vx_{t} \in \sB_2^d(1)$ be a sequence of vectors. If there exists a sequence $0 = \tau_0<\tau_1<\tau_2<\ldots<\tau_z = t$ such that for each $1\leq  \zeta \leq z$, there exists $\vmu_{\zeta}\in  \mathbb{B}_{2}^d(1)$ such that
	\begin{align}
	\sum_{i=1}^{\tau_{\zeta}}f(\vx_{i}\vmu_{\zeta})+\ell^2 >4(d+2)^2\times \left( \sum_{i=1}^{\tau_{\zeta-1}}f(\vx_{i}\vmu_{\zeta}) +\ell^2 \right),\label{eq:mul1}
	\end{align}
	then $z\leq O(d\log^2(dt/\ell))$.
\end{lemma}

We present the proofs of Lemma~\ref{lemma:gpl_ho} 
and \ref{lemma:ener_bound} respectively in Section~\ref{sec:pfll1} and \ref{sec:pfll2}. Given these two lemmas, we continue analysis as below.

Let $\tau_0 = 0$ and for $i \ge 1,$ we let  
\begin{align*}
\tau_{i} =\min \{t+1\} \cup \left\{\tau ~\Bigg\vert~ \exists \tau_{i-1}\leq\tau'<\tau, \sum_{j=1}^{\tau}f(\vx_{j} \vmu_{\tau'})+\ell^2 >4(d+2)^2\left(\sum_{j=1}^{\tau'}f(\vx_{j} \vmu_{\tau'})+\ell^2 \right) \right\}.
\end{align*}
Let $k = \min \{i \mid \tau_i = t+1\}.$ Then $k$ is well-defined and $k \le O(d \log^2(dt))$ by Lemma~\ref{lemma:ener_bound}. Furthermore, for any $\kappa < k$ and any $\tau_\kappa \le i_1 < i_2 < \tau_{\kappa + 1},$ we have 
\begin{align}
    \sum_{j=1}^{i_2}f(\vx_{j}  \vmu_{i_1})+\ell^2 \leq 4(d+2)^2\left( \sum_{j=1}^{i_1}f(\vx_{j}  \vmu_{i_1}) +\ell^2\right).\label{eq:seqbound-3001}
\end{align}
Now we are ready to prove Lemma~\ref{lemma:gpl}.
We have 
\begin{align}
    \sum_{i=1}^t \min\left\{ \frac{f(\vx_i \vmu_i) }{\sum_{j=1}^{i-1} f(\vx_j \vmu_i)  +\ell^2} ,1\right\}& \leq 2   \sum_{i=1}^t \frac{f (\vx_i\vmu_i) }{\sum_{j=1}^{i} f(\vx_{j}  \vmu_i)  +\ell^2}\nonumber
\\ &     \leq 8(d+2)^2\sum_{\kappa =1}^{k}\left(\sum_{i=\tau_{\kappa -1}}^{\tau_{\kappa}-1} \frac{f(\vx_{i}  \vmu_i) }{ \sum_{j=1}^{\tau_{\kappa}-1}f(\vx_{j}  \vmu_i)+\ell^2 }\right)\label{eq:seqbound-4001}
    \\ & \leq 8(d+2)^2\sum_{\kappa=1}^{k}\left(\sum_{i=\tau_{\kappa-1}}^{\tau_{\kappa}-1} \frac{f(\vx_{i} \vmu_i) }{ \sum_{j=\tau_{\kappa-1}}^{\tau_{\kappa}-1}f(\vx_{j}  \vmu_i)+\ell^2 }\right), \nonumber \\
    &\le k \times O(d^2)  \times O(d \log(t/\ell)) \le O(d^4 \log^3(dt)), \label{eq:seqbound-6001}
\end{align}

where \eqref{eq:seqbound-4001} uses \eqref{eq:seqbound-3001} and \eqref{eq:seqbound-6001} uses Lemma~\ref{lemma:gpl_ho}.

\subsection{Proof of Lemma~\ref{lemma:gpl_ho}}\label{sec:pfll1}
\textbf{Restatement of Lemma~\ref{lemma:gpl_ho}} \emph{For any $\vx_1.\vx_2,\ldots,\vx_{t}\in \mathbb{B}_2^d(1)$ and $\vmu_1,\vmu_2,\ldots,\vmu_n\in \mathbb{B}_2^d(1)$, we have that}
\begin{align}
     \sum_{i=1}^{t} \min \left\{\frac{f(\vx_i \vmu_i)}{ \sum_{j=1}^{t} f(\vx_j \vmu_i) +\ell^2},1 \right\} \leq O(d\log(Cdt/\ell)).\label{eq:gpl_ho1}
\end{align}

\begin{proof}
Let $S_{t}$ be the permutation group over $[t]$. We claim that if 
\begin{align}
  \sum_{i=1}^{t}\frac{f\vx_{i}\vmu_i)}{\sum_{j=1}^t f(\vx_{j} \vmu_i)+\ell^2} = \max_{\xi\in S_{t}}\sum_{i=1}^{t}\frac{f(\vx_{\xi(i)} \vmu_i)}{\sum_{j=1}^t f2(\vx_{\xi(j)} \vmu_i,)+\ell^2},\label{eq:kkk1}
\end{align}
then there exists some $i$ such that $(\vx_{i}  \vmu_i)^2\geq (\vx_{j} \vmu_i)^2$ for any $j\in [t]$. Otherwise, we construct a directed graph $G = (V,E)$ where $V =[t]$ and edge $(i, j)$ with $i\neq j$ is in $E$  if and only if $(\vx_{j}  \vmu_i)^2\geq (x_{j'}  \vmu_i)^2$ for any $j'\in [t]$. Let $d(i)$ be the out degree of $i$. By assuming $\{(\vx_{i}  \vmu_i)^2\geq (\vx_{j}^\top\vmu_i)^2,\forall j\in [t]\} $ fails to hold, we learn that $d(i)\geq 1$ for every $i$, so there exists a circle $(i_1,i_2,\ldots,i_{k})$ in $G$. Consider the permutation $\xi$ such that $\xi(i_j) = i_{j+1}$ for $j \in [k]$ (with $i_{k+1}:=i_1$) and $\xi(i)=i$ for $i\notin \{i_1,\ldots,i_k\}$. By definition, we have  $(\vmu_{i_j}  \vx_{\xi(i_j)})^2> (\vmu_{i_j}  \vx_{i_j})^2$ for $j\in [k]$, which implies that $f(\vmu_{i_j}  \vx_{\xi(i_j)})>f(\vmu_{i_j} \vx_{i_j})$ for $j\in [k]$. Therefore
\begin{align}
     \sum_{i=1}^{t}\frac{f(\vx_{i}\vmu_i)}{\sum_{j=1}^t f(\vx_{j}\vmu_i)+\ell^2}<  \sum_{i=1}^{t}\frac{f(\vx_{\xi(i)}\vmu_i)}{\sum_{j=1}^t f(\vx_{j} \vmu_i)+\ell^2} =\sum_{i=1}^{t}\frac{f(\vx_{\xi(i)} \vmu_i)}{\sum_{j=1}^t f(\vx_{\xi(j)} \vmu_i)+\ell^2}, \notag
\end{align}
which leads to contradiction.

We assume that \eqref{eq:kkk1} holds, otherwise we can bound an upper bound of the original quantity. Therefore, we can find an index  $i$  such that $(\vx_{i}  \vmu_i)^2\geq (\vx_{j}^\top\vmu_i)^2$ for any $j\in [t]$. Without loss of generality, we assume $i=1$. 
Because $\frac{f(x)}{x^2}$ is decreasing in $x,$ so we have
\begin{align*}
\frac{f(\vx_{1}\vmu_1) }{(\vx_1  \vmu_1)^2}\leq\frac{f(\vx_{j}  \vmu_1) }{(\vx_j  \vmu_1)^2}
\end{align*}
for any $j\in [t]$, which implies
\begin{align}
\frac{f(\vx_{1} \vmu_1)}{\sum_{j=1}^t f(\vx_{j} \vmu_1)+\ell^2}= \frac{ (\vx_1 \vmu_1)^2}{ \left(\sum_{j=1}^t f(\vx_{j} \vmu_1)+\ell^2 \right) \cdot   \frac{(\vx_1  \vmu_1)^2}{f(\vx_{1} \vmu_1) }  }\leq \frac{(\vx_1\vmu_1)^2 }{ \sum_{j=1}^{t}(\vx_j  \vmu_1)^2 + \ell^2 }.\label{eq:kkk2}
\end{align}
Therefore, we have
\begin{align}
\sum_{i=1}^{t}\frac{f(\vx_{i} \vmu_i)}{\sum_{j=1}^t f(\vx_{i} \vmu_i)+\ell^2} &\leq \frac{(\vx_1  \vmu_1)^2 }{ \sum_{j=1}^{t}(\vx_j  \vmu_1)^2 + \ell^2 } + \sum_{i=2}^{t}\frac{f(\vx_{i} \vmu_i)}{\sum_{j=1}^t f(\vx_{i} \vmu_i)+\ell^2} \nonumber
\\ & \leq \frac{(\vx_1   \vmu_1)^2 }{ \sum_{j=1}^{t}(\vx_j   \vmu_1)^2 + \ell^2 } + \sum_{i=2}^{t}\frac{f(\vx_{i} \vmu_i)}{\sum_{j=2}^t f(\vx_{i} \vmu_i)+\ell^2}.\label{eq:kkk}
\end{align}
Similarly, we can show that there exists a permutation $\xi^*\in S_{t} $ such that
\begin{align}
\sum_{i=1}^{t}\frac{f(\vx_{i} \vmu_i)}{\sum_{j=1}^t f(\vx_{j} \vmu_i)+\ell^2} \leq \sum_{i=1}^{t}\frac{(\vx_{\xi^*(i)}^{\top }\vmu_i)^2 }{ \sum_{j=i}^{t}(\vx_{\xi^*(j)}^{\top }\vmu_i)^2 + \ell^2 } .\label{eq:kkk3}
\end{align}
Finally, by Lemma~\ref{lemma:bound_m2_aux2}, we have that
\begin{align}
    \sum_{i=1}^{t}\frac{(\vx_{\xi^*(i)}   \vmu_i)^2 }{ \sum_{j=i}^{t}(\vx_{\xi^*(j)}   \vmu_i)^2 + \ell^2 } =    \sum_{i=1}^{t} \min \left\{\frac{(x_{\xi^*(i)}   \vmu_i)^2 }{ \sum_{j=i}^{t}(x_{\xi^*(j)}   \vmu_i)^2 + \ell^2 }, 1\right\} \leq O(d\log(t/\ell)).\nonumber
\end{align}
\end{proof}

\subsection{Proof of Lemma~\ref{lemma:ener_bound}}\label{sec:pfll2}
\textbf{Restatement of Lemma~\ref{lemma:ener_bound}}\emph{ 
Let $\vx_{1},\vx_{2},\ldots,\vx_{t} \in \sB_2^d(1)$ be a sequence of vectors. If there exists a sequence $0 = \tau_0<\tau_1<\tau_2<\ldots<\tau_z = t$ such that for each $1\leq  \zeta \leq z$, there exists $\vmu_{\zeta}\in  \mathbb{B}_{2}^d(1)$ such that
}
\begin{align}
	\sum_{i=1}^{\tau_{\zeta}}f(\vx_{i}\vmu_{\zeta})+\ell^2 >4(d+2)^2\times \left( \sum_{i=1}^{\tau_{\zeta-1}}f(\vx_{i}\vmu_{\zeta}) +\ell^2 \right),\label{eq:mul11}
	\end{align}
	\emph{
	then $z\leq O(d\log^2(dt/\ell))$.}
\begin{proof}
If $f(1)\leq \ell^2/t$, then the conclusion holds trivially because $0\leq f(x)\leq f(1)\leq \ell^2/t$ for all $x\in [-1,1]$. Suppose $f(1)> \ell^2/t$. Since $\frac{f(x)}{x^2}\leq \frac{f(y)}{y^2}\leq 1$ for all $x^2\geq y^2$, we have that for $0<\lambda\leq 1$ and any $x\in \mathbb{R}$, $f(\lambda x)\geq \lambda^2 f(x)$.

Let $\ve_i = [0,\ldots,1,\ldots,0] $ be the one-hot vector whose only $1$ entry is at its $i$-th coordinate. Noting that $f(x)\leq x^2$, $|\vx_i \vmu_{\zeta}|\leq \|\vmu_{\zeta}\|_2$ and
\begin{align}
    	\sum_{i=1}^{\tau_{\zeta}}f(\vx_{i}\vmu_{\zeta}) >4(d+2)^2\times \left( \sum_{i=1}^{\tau_{\zeta-1}}f(\vx_{i}\vmu_{\zeta}) +\ell^2 \right)-\ell^2\geq 4d^2 \ell^2 \nonumber
\end{align}
we have that $|\vmu_{\zeta}|_2\geq \sqrt{\frac{4d^2\ell^2}{t}} $. Define $E_{\tau}(\vmu) = \sum_{i=1}^{t}f(\vx_{t}\vmu) + \frac{\ell^2}{d}\sum_{i=1}^d f(\ve_{i}\vmu)$. Then $E_{\tau}(\vmu)$ is convex in $\vmu$ because $f(x)$ is convex in $x$.
By definition, we have that
\begin{align}
    E_{\tau}(\vmu)\leq   \sum_{i=1}^{\tau}f(\vx_{i}\vmu) +\ell^2. \nonumber
\end{align}
By \eqref{eq:mul11}, we have that
\begin{align}
    E_{\tau_{\zeta}}(\vmu_{\zeta}) \geq \sum_{i=1}^{\tau_{\zeta}}f(\vx_i \vmu_{\zeta}) \geq 4d^2 \left( \sum_{i=1}^{\tau_{\zeta-1}}f(\vx_{i}\vmu_{\zeta}) +\ell^2 \right)\geq 4d^2 E_{\tau_{\zeta-1}}(\vmu_{\zeta}).\label{eq:acccc1}
\end{align}

Define 
	\begin{align*}
	    \Lambda = \left\{i \in \sZ : \left\lfloor\log_{2}(d\ell^4/t^2)+2\right\rfloor  \le i \le 2\left\lfloor\log_2t+2 \right\rfloor \right\}.
	\end{align*}
We consider the convex set $D_{\tau,i} = \{\vmu: E_{\tau}(\mu)\leq 2^i \}$ for $i\in \Lambda$. Let $\zeta$ be fixed. 
Because $\|\vmu_{\zeta}\|\geq \sqrt{\frac{4d^2\ell^2}{t}}$ and $\sup_{i}f(\ve_i\vmu)\geq 
\frac{4d\ell^2}{t}\cdot f(1)\geq \frac{4d\ell^4}{t^2}$, we have that $\frac{4d\ell^4}{t^2}\leq E_{\tau}(\vmu_{\zeta}) \leq t+\ell^2\leq t+1$ for any $1\leq \tau \leq t$. Then we can find $i_{\zeta}\in \Lambda$ such that $E_{\tau_{\zeta-1}}(\vmu_{\zeta})\in (2^{i_{\zeta}-1},2^{i_{\zeta}} ]$, which means that $\vmu_{\zeta}\in D_{\tau_{\zeta-1},i_{\zeta}}$. Note that for $0\leq \lambda\leq 1$, $f(\lambda x)\geq \lambda^2 f(x)$ for any $x$, it then follows that $E_{t}(\lambda \vmu)\geq \lambda^2 E_{t}(\vmu)$ for any $t,\vmu$. Choosing $\lambda = \frac{1}{d}$, we have that $E_{\tau_{\zeta}}(\frac{\vmu_{\zeta}}{d})\geq\frac{1}{d^2}E_{\tau_{\zeta}}(\vmu_{\zeta})\geq 4 E_{\tau_{\zeta-1}}(\vmu_{\zeta})\geq 2^{i_{\zeta}}$. Therefore, $\frac{\vmu_{\zeta}}{d}\notin D_{\tau_{\zeta},i_{\zeta}}$. In words, the intercept of $D_{\tau_{\zeta},i_{\zeta}}$ in the direction $\vmu_{\zeta}$ is at most $1/d$ times of that of $D_{\tau_{\zeta-1},i_{\zeta}}$.

	 Note that $D_{t,i}$ is decreasing in $t$ for any $i$, so by Lemma~\ref{lemma:d-multipling}, we have 
	 \begin{align}
	 \mathrm{Volume}(D_{\tau_{\zeta},i_{\zeta}})\leq \frac{6}{7}\mathrm{Volume}(D_{\tau_{\zeta-1},i_{\zeta}}). \notag
	 \end{align}
Also note that $\mathrm{Volume}(D_{0,i})\leq  (\frac{2t}{\ell})^{d}$ and $\mathrm{Volume}(D_{t,i})\geq (\frac{1}{dt^3})^{d}$, so we conclude that $z\leq d\lvert\Lambda\rvert\log_{7/6}(2dt^4/\ell)  \le O(d\log^2(td/\ell))$. 

\end{proof}

\subsection{Other Lemmas and Proofs}

\begin{lemma}\label{lemma:bound_m2_aux2} Fix $\ell \in (0,1].$
Let $\vx_{1},\vx_{2},\ldots,\vx_{t}\in \mathbb{B}_2^d(1)$ and $\vmu_1,\vmu_2,\ldots,\vmu_t\in \mathbb{B}_2^d(1)$ be two sequences of vectors. Then we have 
\begin{align}
       \sum_{i=1}^{t}\mathbb{I}\bigg\{(\vx_{i}\vmu_i)^2 >\sum_{j=1}^{i-1}(\vx_{j}   \vmu_i)^2+\ell^2 \bigg\}\leq  \sum_{i=1}^t \min\bigg\{ \frac{ (\vx_{i}\vmu_{i})^2}{\sum_{j=1}^{i-1}(\vx_{j}  \vmu_{i})^2 +\ell^2 },1\bigg\}\leq O(d\log\frac t\ell).\label{eq:spe1}
\end{align}
\end{lemma}
\begin{proof}
The first inequality in \eqref{eq:spe1} holds clearly.
To prove the second inequality, we define $\mU_0 = \ell^2 \mI$ and $\mU_i =\ell^2 \mI + \sum_{j=1}^i \vx_{j}\vx_{j}^\top$ for $i\geq 1.$ Note that
\begin{align}
   \frac{ (\vx_{i}  \vmu_{i})^2}{\sum_{j=1}^{i-1}(\vx_{j}   \vmu_{i})^2+\ell^2 }& \leq   \frac{(\vx_{i}   \vmu_{i})^2}{\vmu_{i}^{\top}   \mU_{i-1}\vmu_{i}} \leq \vx_{i}^{\top}  \mU_{i-1}^{-1}\vx_{i},\nonumber
\end{align}
where the first inequality is because $\lVert \vmu_i\rVert_2 \leq 1$ and  the second inequality uses the Cauchy's inequality, so we have
\begin{align}
    \sum_{i=1}^{t} \min\bigg\{\frac{ (\vx_{i}  \vmu_{i})^2}{\sum_{j=1}^{i-1}(\vx_{j}\vmu_{i})^2+\ell^2 },1\bigg\} \leq \sum_{i=1}^{t}\min\left\{\vx_{i}^{\top}  \mU_{i-1}^{-1}\vx_{i},1 \right\} \leq 2d\ln(t/\ell^2) \leq 4d\ln(t/\ell), \notag 
\end{align}
where the second-to-third inequality uses the elliptical potential lemma.
\end{proof}

\begin{lemma}\label{lemma:d-multipling} 
Given $x\in \mathbb{R}^d$, we use $(u(x),l(x))$ to denote  the polar coordinate of $x$ where $\|\mu(x)\|_2 = \frac{x}{\|x\|_2}$ is the direction and $l(x)=\|x\|_2$. We also use $(u,\ell)$ to denote the unique element $x$ in $\mathbb{R}^{d}$ such that $(u(x),l(x)) = (u,\ell)$.
	Let $D$ be a bounded symmetric convex subset of $\mathbb{R}^d$ with $d\geq 2$. Given any direction $\mu\in \partial \mathbb{B}_{d}$,  there exists a unique $l(u)\in \mathbb{R}$ such that $(u,l(u)), (-u,l(u)) \in \partial D$ are on its boundary. Let $D'$ be a bounded symmetric convex subset of $\mathbb{R}^d$ containing  $D\subseteq D'$ such that  $(u,d\cdot l(u))\in D'$ for some direction  $u\in \partial \mathbb{B}_d$. Then we have that
	\begin{align*}
	\mathrm{Volume}(D')\geq \frac{7}{6}\mathrm{Volume}(D).
	\end{align*}
\end{lemma}
\begin{proof}
	Let $A = (u, l(u))$ and $B =(u,d\cdot l(u)) $. Since $A$ is on the boundary of $D$, we can find a hyperplane $h_1$ such that $A\in h_1$ and $h_1$ is tangent to $D$. Let $h_2$ be the parallel hyperplane of $h_1$ containing the origin $O\in h_2$.  Define
	\begin{align*}
	H = \left\{ x\in \mathbb{R}^d ~\Bigg|~ d(x,h_1)+d(x,h_2) = d(h_1,h_2), \exists y\in D ,\lambda\in \mathbb{R}, (B-y)=\lambda (B-x) \right\}
	\end{align*}
 It is obvious that $\mathrm{Volume}(H)\geq \frac{1}{2}\mathrm{Volume}(D)$ since for each $x\in D$ lying between $h_1$ and $h_2$, $x\in H$.  Define 
 \begin{align*}
 U = \left\{ x\in \mathbb{R}^d~\Bigg|~d(x,h_2) = d(x,h_1)+d(h_1,h_2) , \exists y\in H, \lambda\in [0,1], x = \lambda y + (1-\lambda)B   \right\}.
 \end{align*}
We claim that 
 \begin{align}
 \mathrm{Volume}(U) =\left(1-\frac{1}{d}\right)^d \mathrm{Volume}(U\cup H) =\left(1-\frac{1}{d}\right)^d  \left(\mathrm{Volume}(U) +\mathrm{Volume}(H)  \right).\label{eq:volume-1000}
 \end{align}
 To see the first equality, we note that $U$ and $U \cup H$ are both $d$-dimensional pyramids. It then follows from the volume formula and the relation $d(B, O) = d \times d(A, O).$ The second equality is because by their definitions, $U, H$ are separated by the hyperplane $h_1,$ and thus they are disjoint. Finally, by \eqref{eq:volume-1000}, we have 
 \begin{align}
     \mathrm{Volume}(D')&\geq\mathrm{Volume}(U) + \mathrm{Volume}(H) = (1+ \frac{1}{1-(1-1/d)^d}) \mathrm{Volume}(H) \notag \\
    &\ge \frac{1}{2} (1+ \frac{1}{(1-(1-1/d)^d)}) \mathrm{Volume}(D) \geq \frac{7}{6}\mathrm{Volume}(D).\nonumber 
 \end{align}
\end{proof}

\section{Missing Proofs in Section~\ref{sec:bandit}}
\label{sec:proof_bandit}

\subsection{Application of the General Potential Lemma}
As an application of Lemma~\ref{lemma:gpl} on linear bandit and linear RL, we have the lemma as below

\begin{lemma}\label{lemma:bound_m2}
	Fix $\ell\in (0,1].$ Let $\vx_{1},\vx_{2}, \ldots,\vx_{t}\in \mathbb{B}^d_2(1)$ be a sequence of vectors, and $\vmu_1,\vmu_2, \ldots, \vmu_t\in \mathbb{B}_{2}^d(1)$ be another sequence of vectors. Then we have
	\begin{align}
	\sum_{i=1}^t \frac{\clip^2(\vx_i \vmu_i, \ell) }{\sum_{j=1}^{i-1} \clip(\vx_j\vmu_i, \ell)\vx_j^\top \vmu_i  +\ell^2}\leq O(d^4\log^3(dt)). \label{eq:seqbound-1000}
	\end{align}
\end{lemma}
\begin{proof}
Let
\begin{align}
	f_{\ell}(x) =
	\begin{cases} x^2, &  \lvert x \rvert \leq \ell,\\
		2 \ell x-\ell^2, &  x>\ell,\\
		-2\ell x-\ell^2, & x<-\ell
	\end{cases} \notag
\end{align}
be a convex relaxation of the function $x\mapsto \clip(x,\ell)x$.
It is easy to see that $f_\ell(x)$ is convex in $x$ and for any $x\in \sR, \ell>0$,
\begin{align}
	\clip(x,\ell)x\leq f_{\ell}(x)\leq 2\clip(x,\ell)x\leq 2x^2. \label{eq:convex}
\end{align}
Let $h(x)=\frac{f_{\ell}(x)}{2}$.
It is easy to see that if $x^2 \geq y^2$, $\frac{h(x)}{x^2}  = \frac{\clip(x,l)}{2x}\leq \frac{\clip(y,l)}{2y} = \frac{h(y)}{y^2}\leq 1$. 
  By Lemma~\ref{lemma:gpl} with $f(x)=h(x)$, we have that
  \begin{align}
      \sum_{i=1}^t \frac{h(\vx_i \vmu_i) }{\sum_{j=1}^{i-1} h(\vx_j\vmu_i)+\ell^2}\leq O(d^4\log^3(dt)). \nonumber
  \end{align}
  By \eqref{eq:convex}, we obtain that
  \begin{align}
  	\sum_{i=1}^t \frac{\clip^2(\vx_i \vmu_i, \ell) }{\sum_{j=1}^{i-1} \clip(\vx_j\vmu_i, \ell)\vx_j^\top \vmu_i  +\ell^2} &\leq   	\sum_{i=1}^t \frac{\clip(\vx_i \vmu_i, \ell)\vx_i\vmu_I }{\sum_{j=1}^{i-1} \clip(\vx_j\vmu_i, \ell)\vx_j^\top \vmu_i  +\ell^2}  \nonumber
  	\\ & \leq  4\sum_{i=1}^t \frac{h(\vx_i \vmu_i) }{\sum_{j=1}^{i-1} h(\vx_j\vmu_i)+\ell^2}\nonumber
  	\\ &\leq O(d^4\log^3(dt)).\nonumber
  \end{align}
  The proof is completed.
\end{proof}

\subsection{Proof of Theorem~\ref{thm:oful}}
\label{app:thm-m0m1}

\subsubsection{Optimism} The equation~\eqref{eq:bandit-confball} accounts for the main novelty of our algorithm. We note that our confidence set is different from all previous ones \citep{dani2008Stochastic,abbasi2011improved}.
Our confidence set is built based on the following new inequality, which may be of independent interest.

With Lemma~\ref{thm:m0m1} in hand, we can easily prove that the optimal $\vtheta^*$ is always in our confidence set with high probability. The proof details can be found in Appendix~\ref{app:lem-bandit-hpe-proof}.
\begin{lemma}\label{lem:bandit-hpe}  With probability at least $1 - O(\delta \log K),$ we have $\vtheta^* \in \Theta_k$ for all $k \in [K].$
\end{lemma}

\subsubsection{Bounding the Regret}
We bound the regret under the event specified in Lemma~\ref{lem:bandit-hpe}. We have 
	\begin{align*}
\Reg^K =& \sum_{k = 1}^K (\max_{\vx \in \gA_k} \vx \vtheta^* - \vx_k \vtheta^*)\\
\le &\sum_{k = 1}^K \left(\max_{\vx \in \gA_k, \vtheta \in \Theta_k} \vx \vtheta - \vx_k \vtheta^*\right) 
\le \sum_{k = 1}^K \vx_k \left(\vtheta_k - \vtheta^*\right)
= \sum_{k} \vx_k \vmu_k,
\end{align*}
where second inequality follows from Lemma~\ref{lem:bandit-hpe}. Therefore, it suffices to bound $\sum_{k} \vx_k \vmu_k$, for which we have the following lemma.
\begin{lemma} \label{lem:bandit-conf-sum} With probability $1 - O(\delta \log K),$ we have
\begin{align*}
    \sum_{k} \vx_k \vmu_k \le O\Bigg( d^{4.5} \big(\log^4 dK\big) \big(\log \frac{dK}{\delta}\big) \bigg( \sqrt d + \sqrt{\sum_{k=1}^K \sigma_k^2}\bigg)\Bigg).
\end{align*}
\end{lemma}
Since this lemma is one of our main technical contribution, we provide more proof details.

\begin{proof}
First, we define the desired event $\gE = \gE_1 \cap \gE_2,$ where
\begin{align*}
    \gE_1 = \{\forall k \in [K]: \vtheta^* \in \Theta_k\}, \qquad \gE_2 = \bigg\{\sum_{k=1}^{K} \eta_k(\vtheta^*) \le \sum_{k = 1}^K 8 \sigma_k^2 + 4 \ln\frac4\delta\bigg\}.
\end{align*}
By Lemma~\ref{lem:bandit-hpe}, we have $\Pr[\gE_1] \ge 1 - O(\delta).$ By Lemma~\ref{lem:mchernoff}, we have $\Pr[\gE_2] \ge 1 - O(\delta \log K).$
Therefore,  by union bound, we have $\Pr[\gE] \ge 1 - O(\delta \log K).$

Now we bound $\sum_{k} \vx_k \vmu_k$ under the event $\gE$ to prove the lemma. Recall that
\begin{align}
  \Phi_k^{j}(\vmu) = \sum_{v=1}^{k-1} \pj_j(\vx_{v}\vmu) \vx_{v}\vmu + \ell_j^2, \qquad 
   \Psi_k^{j}(\vmu) = \sum_{v=1}^{k-1}  \pj_j^2(\vx_{v} \vmu) \eta_v(\vtheta^*). \label{eq:bandit-phi-psi}
\end{align}

Recall the definition of $\{\mathcal{K}_{j}\}_{j=1}^{L_2+1}$ in Section~\ref{sec:pfs}

To proceed, we need the following claim.
\begin{claim}\label{claim:bandit_multiply_1} We have
\begin{align}
\sum_k \vx_k \vmu_k &= \sum_{k\in \mathcal{K}_{L_2 + 1}} \vx_k \vmu_k + \sum_{j=1}^{L_2}\sum_{k\in \mathcal{K}_j} \vx_k \vmu_k \notag \\
&\le 1 +  \sum_{j=1}^{L_2}\sum_{k\in \mathcal{K}_j} \vx_k \vmu_k \times  \frac{3\sqrt{\Psi_k^{j}(\vmu_k) \iota} + \sqrt{\sum_{v = 1}^{k-1} 2\pj^2_{j}(\vx_{v} \vmu_k) (\vx_v \vmu_k)^2 \iota} + 3  \ell_{j}\iota }{\Phi_k^{j}(\vmu_k)}. \label{eq:bandit-2000}
\end{align}
\end{claim}
We defer the proof of the claim to Appendix~\ref{app:proof-claim} and continue to bound the three terms in \eqref{eq:bandit-2000}. For the second term, we have
\begin{align} \hspace{-2em}
 &\quad  \sum_{j=1}^{L_2}\sum_{k\in \mathcal{K}_j} \vx_k \vmu_k \frac{ \sqrt{\sum_{v = 1}^{k-1} 2\pj^2_{j}(\vx_{v} \vmu_k) (\vx_v \vmu_k)^2\iota}}{\Phi_k^{j}(\vmu_k)} \notag  \\
 &\le \frac{1}{2}  \sum_{j=1}^{L_2}\sum_{k\in \mathcal{K}_j} \vx_k \vmu_k +   \sum_{j=1}^{L_2}\sum_{k\in \mathcal{K}_j} \vx_k \vmu_k  \sI\bigg\{\frac{\sqrt{\sum_{v = 1}^{k-1} 2\pj^2_{j}(\vx_{v} \vmu_k) (\vx_v \vmu_k)^2\iota }}{\Phi_k^{j}(\vmu_k)} > \frac{1}{2}\bigg\}. \label{eq:bandit-3000}
\end{align}
We note that 
\begin{align} \hspace{-2em}
    \sum_{j=1}^{L_2}\sum_{k\in \mathcal{K}_j}\vx_k \vmu_k  \sI\bigg\{\frac{\sqrt{\sum_{v = 1}^{k-1} 2\pj^2_{j}(\vx_{v} \vmu_k) (\vx_v \vmu_k)^2\iota}}{\Phi_k^{j}(\vmu_k)} > \frac{1}{2}\bigg\} &\le  \sum_{j=1}^{L_2}\sum_{k\in \mathcal{K}_j} \vx_k \vmu_k  \sI\left\{ \Phi_k^{j}(\vmu_k) \le 4 \ell_{j}\iota  \right\} \notag \\
    &\le \sum_{j=1}^{L_2}\sum_{k\in \mathcal{K}_j} \vx_k \vmu_k \frac{4 \ell_{j}\iota }{\Phi_k^{j}(\vmu_k)} \notag \\
    &\le \sum_{j=1}^{L_2}\sum_{k\in \mathcal{K}_j} \frac{4  \pj_{j}^2(\vx_k \vmu_k)\iota }{\Phi_k^{j}(\vmu_k)} \notag \\
    &\le O( d^4 \abs{\Lambda_2}\iota \log^3(dK)), \label{eq:bandit-4000}
\end{align}
where the last inequality uses Lemma~\ref{lemma:bound_m2}. Collecting \eqref{eq:bandit-2000},\eqref{eq:bandit-3000} and \eqref{eq:bandit-4000}, we have 
\begin{align*}
   \sum_{k} \vx_k \vmu_k \le 1 +  \sum_{j=1}^{L_2}\sum_{k\in \mathcal{K}_j}3 \vx_k \vmu_k \times \frac{\sqrt{\Psi_k^{j}(\vmu_k)\iota} + \ell_{j}\iota}{\Phi^{j}_k(\vmu_k)} + \frac{1}{2} \sum_{j=1}^{L_2}\sum_{k\in \mathcal{K}_j} \vx_k \vmu_k + O( d^4 \abs{\Lambda_2}\iota\log^3(dK)).
\end{align*}
Solving $\sum_{k} \vx_k \vmu_k$, we obtain
\begin{align}
   \sum_{k} \vx_k \vmu_k &\le O( d^4 \abs{\Lambda_2}\iota \log^3(dK)) +  \sum_{j=1}^{L_2}\sum_{k\in \mathcal{K}_j} 6 \vx_k \vmu_k \times \frac{\sqrt{\Psi_k^{j}(\vmu_k)\iota} +  \ell_{j}\iota}{\Phi^{j}_k(\vmu_k)}  \notag \\
   &\le O(d^4 \abs{\Lambda_2}\iota \log^3(dK)) + \sum_{j=1}^{L_2}\sum_{k\in \mathcal{K}_j} 12 \vx_k \vmu_k \times \frac{\sqrt{\Psi_k^{j}(\vmu_k)\iota}}{\Phi^{j}_k(\vmu_k)}, \label{eq:bandit-6100}
   \end{align}
where \eqref{eq:bandit-6100} uses the last two steps in \eqref{eq:bandit-4000}. The remaining term in \eqref{eq:bandit-6100} can be bounded as 
\begin{align}
    \sum_{j=1}^{L_2}\sum_{k\in \mathcal{K}_j} 12 \vx_k \vmu_k \times \frac{\sqrt{\Psi_k^{j}(\vmu_k)\iota}}{\Phi^{j}_k(\vmu_k)} &\le  \sum_{j=1}^{L_2}\sum_{k\in \mathcal{K}_j} 12 \vx_k \vmu_k \ell_{j} \frac{\sqrt{\sum_{v=1}^{k-1} \eta_v(\vtheta^*)\iota}}{\Phi^{j}_k(\vmu_k)}  \label{eq:bandit-6200}\\
   &\le \sum_{j=1}^{L_2}\sum_{k\in \mathcal{K}_j} \frac{12 \vx_k \vmu_k \ell_{j}}{{\Phi^{j}_k(\vmu_k)}} \sqrt{\sum_{k=1}^{K} \eta_k(\vtheta^*)\iota}  \notag \\
    &\le  O(d^4 \abs{\Lambda_2} \log^3(dK)) \times \sqrt{\sum_{k=1}^{K} \eta_k(\vtheta^*)\iota}\label{eq:bandit-6300} \\
    &\le  O(d^4 \abs{\Lambda_2} \log^3(dK)) \times  \sqrt{\Big(\ln\frac1\delta + \sum_{k=1}^K \sigma_k^2\Big) \iota}, \label{eq:bandit-6400}
\end{align}
where \eqref{eq:bandit-6200} uses the definition of $\Psi_k^j(\cdot),$ \eqref{eq:bandit-6300} again uses the last two steps in \eqref{eq:bandit-4000}, and \eqref{eq:bandit-6400} uses the event $\gE_2.$
\end{proof}

Now we can finish the proof of Theorem~\ref{thm:oful}. We choose $\delta = O((K\log K)^{-1})$. Since on the event $\gE^{C}$, we have  $\Reg^K \leq K.$ 
Therefore, together with the bound on $\gE$ from Lemma~\ref{lem:bandit-conf-sum},
we conclude that the expected regret is bounded by $\E[\Reg^K] \le \widetilde{O}( d^{4.5}  \sqrt{\sum_{k=1}^K \sigma_k^2}+d^5)$.

\begin{proof}
It suffices to prove the theorem for $b = 1,$ because otherwise we can apply $\{X_i/b\}_{i=1}^n$ to the $b=1$ case. By Lemma~\ref{lem:ten} with $\epsilon = 1$ and $\delta < 1/e$, we have 
\begin{align}
\Pr[\abs{\sum_{i = 1}^n X_i} \ge 2 \sqrt{\sum_{i = 1}^n 2 \E[X_i^2 \mid \gF_{i-1}]\ln\frac1\delta} + 4 \ln\frac1\delta ] \le 4 \delta \log_2 n. \label{eq:m0m1-1000}
\end{align}
By Lemma~\ref{lem:mchernoff2}, we have 
\begin{align}
    \Pr[\sum_{i = 1}^n \E[X_i^2\mid \gF_{i-1}] \ge  \sum_{i = 1}^n 8 X_i^2 + 4 \ln\frac4\delta] \le (\lceil \log_2 n \rceil + 1) \delta. \label{eq:m0m1-2000}
\end{align}
Therefore, by a union bound over \eqref{eq:m0m1-1000} and \eqref{eq:m0m1-2000}, we have with probability at least $1 - 6 \delta \log_2 n,$
\begin{align*}
    \abs{\sum_{i = 1}^n X_i} &\le \sqrt{\sum_{i = 1}^n 8 \E[X_i^2 \mid \gF_{i-1}]\ln\frac1\delta} + 4 \ln\frac1\delta \\
    &\le \sqrt{8 \left( \sum_{i = 1}^n 8 X_i^2 + 4 \ln\frac4\delta \right)\ln\frac1\delta} + 4 \ln\frac1\delta \le 8 \sqrt{\sum_{i = 1}^n X_i^2 \ln\frac1\delta} + 16 \ln\frac1\delta,
\end{align*}
which concludes the proof.
\end{proof}

\subsection{Proof of Lemma~\ref{lem:bandit-hpe}}
\label{app:lem-bandit-hpe-proof}
\begin{proof} Let $\delta' = e^{-\iota}.$ We define the desired event $\gE = \bigcap_{k \in [K],j \in \Lambda_2} \gE_k^{j},$ where 
	\begin{align*}
	\gE_k^{j} = \Bigg\{\abs{\sum_{v=1}^k \pj_j(\vx_{v} \vmu) \epsilon_{v}(\vtheta^*)}  \leq \sqrt{ \sum_{v=1}^k  \pj_j^2(\vx_v \vmu) \eta_v(\vtheta^*)\iota} + \ell_j \iota, \forall \vmu \in \gB\Bigg\}.
	\end{align*}
	
	Note that for each $v,$ we have that $\abs{ \pj_j(\vx_{v} \vmu) \epsilon_{v}(\vtheta^*)} \le \ell_j$ and that $(\pj_j(\vx_{v} \vmu) \epsilon_{v}(\vtheta^*))^2 = \pj_j^2(\vx_v \vmu) \eta_v(\vtheta^*),$ so by Theorem~\ref{thm:m0m1}, we have 
	\begin{align*}
	&\Pr[\abs{\sum_{v=1}^k \pj_j(\vx_{v} \vmu) \varepsilon_{v}} 
	\le \sqrt{ \sum_{v=1}^k  \pj_j^2(\vx_v \vmu)\Var(\varepsilon_v \mid \gF_v) \iota} + \ell_j \iota] \\
	&\ge 1 - O\left(e^{-\frac{\iota}{\log_2 \log_2 K}}\right) \\
	&\ge 1 - O\left(\frac{\delta}{K \abs{\gB}\abs{\Lambda_2}} \log K\right),
	\end{align*}
	where $\gF_v$ is as defined in Section~\ref{sec:pre}. Finally, using a union bound over $(\vmu,j,k) \in \gB \times  \Lambda_2 \times [K],$ we have $\Pr[\gE] \ge 1-O(\delta \log K).$ 
\end{proof}

\subsection{Proof of Claim~\ref{claim:bandit_multiply_1}}

\label{app:proof-claim}

\begin{proof}
We elaborate on \eqref{eq:bandit-2000}. We will prove it by showing that the numerator is always greater than the denominator in the fraction in \eqref{eq:bandit-2000}, so each term $\vx_k \vmu_k$ is multiplied by a number greater than $1.$ We have for every $j \in \Lambda_2,$
\begin{align}
\Phi_{k}^{j}(\vmu_{k}) &=    \sum_{v=1}^{k-1} \pj_j(\vx_{v} \vmu_k) \vx_v \vmu_k + \ell_j^2 \notag \\
&\le \abs{\sum_{v=1}^{k-1}  \pj_j(\vx_{v} \vmu_k) \epsilon_v(\vtheta^*)} + \abs{\sum_{v=1}^{k-1}\pj_j(\vx_{v} \vmu_k) \epsilon_v(\vtheta_k)} + \ell_j^2 \label{eq:bandit-1000} \\
&\le \sqrt{\Psi_k^j(\vmu_k) \iota} +  \sqrt{\sum_{v=1}^{k-1}  \pj^2_j(\vx_{v} \vmu_k) \eta_v(\vtheta_k)\iota} +  l_j^2+\frac{1}{HK}+ 2 \ell_j \iota  \label{eq:bandit-1100} \\ 
&\le \sqrt{\Psi_k^j(\vmu_k) \iota} + \sqrt{\sum_{v=1}^{k-1}  \pj^2_j(\vx_{v} \vmu_k) \eta_v(\vtheta^*) \iota} +  \sqrt{\sum_{v=1}^{k-1}  \pj^2_j(\vx_{v} \vmu_k) \abs{\eta_v(\vtheta_k) - \eta_v(\vtheta^*)}\iota} + 3 \ell_j \iota \notag \\
&=  2\sqrt{\Psi_k^j(\vmu_k)\iota} +  \sqrt{\sum_{v = 1}^{k-1} \pj^2_j(\vx_{v} \vmu_k) \abs{\eta_v(\vtheta_k) - \eta_v(\vtheta^*)}\iota} + 3 \ell_j \iota \notag \\ 
&\le  2\sqrt{\Psi_k^j(\vmu_k)\iota} +  \sqrt{\sum_{v = 1}^{k-1} \pj^2_j(\vx_{v} \vmu_k) (\eta_v(\vtheta^*) + 2(\vx_v \vmu_k)^2)\iota} + 3 \ell_j \iota \label{eq:bandit-1200} \\ 
&\le 2\sqrt{\Psi_k^j(\vmu_k) \iota} + \sqrt{\sum_{v = 1}^{k-1} \pj^2_j(\vx_{v} \vmu_k)\eta_v(\vtheta^*) \iota} + \sqrt{\sum_{v = 1}^{k-1} 2\pj^2_j(\vx_{v} \vmu_k) (\vx_v \vmu_k)^2 \iota} + 3 \ell_j \iota \notag, \\
&= 3\sqrt{\Psi_k^j(\vmu_k) \iota} + \sqrt{\sum_{v = 1}^{k-1} 2\pj^2_j(\vx_{v} \vmu_k) (\vx_v \vmu_k)^2 \iota} + 3 \ell_j \iota, \label{eq:bandit-1400}
\end{align}
where \eqref{eq:bandit-1000} uses $\epsilon_v(\vtheta_k) - \epsilon_v(\vtheta^*) = \vx_v (\vtheta_k - \vtheta^*) = \vx_v \vmu_k,$ \eqref{eq:bandit-1100} uses that $\vtheta^*, \vtheta_k \in \Theta_k$, the definition of $\Theta_k$ in \eqref{eq:bandit-confball} and the fact that $\mathcal{B}$ is an $K^{-3}$-net of $\mathbb{B}_2^d(2)$, and \eqref{eq:bandit-1200} uses
\begin{align*} \hspace{-2em}
\abs{\eta_v(\vtheta_k) - \eta_v(\vtheta^*)} = &\abs{(\epsilon_v(\vtheta^*) - \vx_v \vmu_k)^2 - (\epsilon_v(\vtheta^*))^2} \\
\le &2 \abs{\epsilon_v(\vtheta^*)} \vx_v \vmu_k + (\epsilon_v(\vtheta^*))^2 \le (\vx_v \vmu_k)^2 + 2 (\epsilon_v(\vtheta^*))^2.
\end{align*}
Since \eqref{eq:bandit-1400} holds for every $j \in \Lambda_2,$ it holds for $j = j_k,$ and thus \eqref{eq:bandit-2000} follows. 
\end{proof}

\section{Missing Proofs in Section~\ref{sec:rl}}
\label{sec:proof_rl}
\subsection{Proof of Theorem \ref{thm:main_rl}}
Before introducing our proof, we make some definitions. We let $\vtheta^m_{k, h} = \argmax_{\vtheta \in \Theta_k} |\vx_{k,h}^m(\vtheta - \vtheta^*)|$ and $\vmu^m_{k,h} = \vtheta^m_{k,h} - \vtheta^*$.
Recall that $\gT^{m,i}_k$ is defined in Algorithm~\ref{alg:main}. We define
\begin{align}
\Phi_k^{m, i, j}(\vmu) = \sum_{(v, u) \in \gT^{m,i}_k} \pj_j(\vx^m_{v,u}\vmu) \vx^m_{v,u}\vmu + \ell_j^2, ~ 
\Psi_k^{m,i,j}(\vmu) = \sum_{(v, u) \in \gT^{m,i}_k}  \pj_j^2(\vx_{v,u}^m \vmu) \eta_{v,u}^m. \label{eq:rl-phi-psi}
\end{align}
Note that our definitions in \eqref{eq:rl-phi-psi} are similar to those for linear bandits in \eqref{eq:bandit-phi-psi}. The main differences are: 1) we define $\Phi(\cdot), \Psi(\cdot)$ also for higher moments, as indicated by the index $m$ in their superscripts; 2) we add the variance layer, so that we only use samples from $\gT^{m,i}$; 3) since we can now estimate variance, we use the upper bound of estimated variance in lieu of the empirical variance.  
For $h \in [H + 1],$ we further define 
\begin{align}
I^k_h = \sI\{\forall u \le h, m, i, j :  \Phi_{k,u}^{m,i,j}(\vmu_{k,u}^m) \le 4(d+2)^2 \Phi_k^{m,i,j}(\vmu_{k,u}^m)\},
\end{align}
where $I^k_h = 1$ indicates that for every $u \le h$, the confidence set using data prior to the time step $(k,u)$ can be properly approximated by the confidence set with data prior to the episode $k$. We define $I^k_h$ in this way to ensure that it is $\gF^k_{h}$-measurable. The following lemma ensures that $Q_h^k$ is optimistic with high probability. Its proof is deferred to Appendix~\ref{app:lemma-optimism}.
\begin{lemma}\label{lemma:optimism} $\Pr[\forall k,h,s,a : Q^k_h(s,a)\ge Q_h^*(s,a)] \ge \Pr[\forall k  \in [K]: \vtheta^* \in \Theta_k] \ge 1- O(\delta).$
\end{lemma}
When the event specified in Lemma~\ref{lemma:optimism} holds, the regret can be decomposed as
\begin{align*}
	\Reg^K &= \sum_{k = 1}^K \left(V_1^*(s_1^k) - V_1^{\pi_k}(s_1^k) \right) \le \sum_{k = 1}^K \left( V_1^k(s_1^k) - V_1^{\pi_k}(s_1^k) \right)  \le \cReg_1 + \cReg_2 + \Reg_3 + \sum_{k,h} (I_h^k - I_{h+1}^k),
\end{align*}
where 
\begin{align*}
\hspace{-2em} \cReg_1 = \sum_{k,h} (P_{s_h^k, a_h^k}V_{h + 1}^k -& V_{h + 1}^k(s_{h+1}^k)) I^k_h,\qquad
 \cReg_2 = \sum_{k,h} \big(V_h^k(s_h^k) - r_h^k - P_{s_h^k, a_h^k} V_{h+ 1}^k \big)I_h^k,\\ 
  &\Reg_3 = \sum_{k = 1}^K \big(\sum_{h = 1}^H r_h^k - V_1^{\pi_k}(s_1^k)\big).
\end{align*}
Next we analyze these terms. First, we observe that $\Reg_3$ is a sum of a martingale difference sequence, so by  Lemma~\ref{lem:azuma}, we have $\Reg_3 \le O(\sqrt{K \log(1/\delta)})$ with probability at least $1-\delta$.
Next, we use the following lemma to bound $\sum_{k,h} (I_h^k - I_{h+1}^k)$. We defer its proof to Appendix~\ref{app:lem-boundI-proof}.
\begin{lemma}\label{lem:boundI} $\sum_{k,h} (\ind^k_h-\ind^k_{h+1})\leq O(d\log^5(dHK)).$
\end{lemma}

To bound $\cReg_1$ and $\cReg_2$, we need to define the following quantities.
First, we denote $\cvx_{k,h} = \vx_{k,h} I_h^k$ and $\ceta_{k,h}^m = \eta_{k,h}^m I_h^k$.
Next, for $m \in \Lambda_0$, we define
\begin{align*}
\cR_m = \sum_{k,h}\left| \cvx_{k,h}^m \vmu_{k,h}^m \right|, \qquad \cM_m = \sum_{k,h} \left( P_{s_h^k, a_h^k}(V_{h + 1}^k)^{2^{m}} - (V_{h + 1}^k(s_{h+1}^k))^{2^{m}}\right) I_h^k.
\end{align*}
Intuitively, $\cR_m$ represents the ``regret'' of $2^m$-th moment prediction and $\cM_m$ represents the total variance of $2^m$-th order value function.
We have $\cReg_1 = \cM_0$ by definition and 
and using that 
\begin{align*}
Q_h^k(s, a) - r(s, a) - P_{s, a} V_{h + 1}^k \le \max_{\vtheta \in \Theta_k} \vx_{k,h}^0(\vtheta - \vtheta^*),
\end{align*}
we have $\cReg_2 \le \cR_0.$
So it suffices to bound $\cR_0 + \cM_0,$ which is done by the following lemma.
\begin{lemma}\label{lem:cr0_cm0}
With probability at least $1-\delta$, we have 
	\[\cR_0 + \abs{\cM_0}\le O\left( d^{4.5}\sqrt{ K \log^5(dHK)\log(1/\delta)}+d^9\log^6(dHK)\log(1/\delta)\right).\]
\end{lemma}
Lemma~\ref{lem:cr0_cm0} is the main technical part of our result in Section~\ref{sec:rl}, so we sketch its proof in the next subsection. With the lemma in hand, we have with probability $1-\delta$ that $\Reg^K \le \widetilde{O}( d^{4.5} \sqrt{ K }+d^9).$ Finally, We conclude the proof to Theorem~\ref{thm:main_rl} by choosing $\delta = 1/K$ and noting that $\Reg^K \le K.$

\subsection{Bounding $\cR$ and $\cM$}
We sketch the proof for Lemma~\ref{lem:cr0_cm0}.
The first step to bound $\cR_m$ is to relate it to the variance $\ceta^m$.
\begin{lemma} \label{lem:sumci} 
With probability at least $1-\delta$, we have 
	$\cR_m \le O(d^{4}\sqrt{\sum_{k, h} \ceta_{k,h}^m \iota \log^7(dHK) }  +  {d^6 \iota \log^5(dHK)}).$
\end{lemma}
We defer the proof to Appendix~\ref{app:lem-sumci-proof}. The proof is spiritually similar to proof of Lemma~\ref{lem:bandit-conf-sum}.
The main difference is that we use the peeling technique to the magnitude of the variance.

Based on Lemma~\ref{lem:sumci}, we use the following recursion lemma to relate $\cR_m,\cM_m$ to $\cR_{m+1},\cM_{m+1}$. We defer the proof to Appendix \ref{app:lem-recurrsion-proof}. It mainly uses similar ideas in \citet{zhang2020reinforcement}.

 \begin{lemma}[Recursions]\label{lem:recurrsion} 
 With probability at least $1-\delta$, we have 
 \begin{align*}
 	\cR_m &\le O\left(d^4\sqrt{( \cM_{m+1}+  2^{m + 1}(K + \cR_0) +  \cR_{m+1} + \cR_m)  \iota \log^7(dHK)} +  {d^6 \iota \log^5(dHK)} \right), \\
     \abs{\cM_m} &\le O\left(\sqrt{(\cM_{m+1} + O(d \log^5(dHK)) +  2^{m + 1}(K + \cR_0)) \log(1/\delta)} + \log(1/\delta)\right).
 \end{align*}
\end{lemma}

Finally, we can prove Lemma~\ref{lem:cr0_cm0} by collecting Lemma~\ref{lem:sumci},\ref{lem:recurrsion} and using a technical lemma about recursion (Lemma~\ref{lemma:sequence2}). The details are in Appendix~\ref{app:lem-cr0-cm0-proof}.

\subsection{Proof of Lemma~\ref{lemma:optimism}}

\label{app:lemma-optimism}

\begin{proof} The lemma consists of two inequalities. 
	The first inequality is proved using backward induction, where the induction step is given as 
	\begin{align*}
		Q_h^k(s,a)  &= \min \{1, r(s,a)+ \max_{\vtheta\in \Theta_k}\sum_{i=1}^d \theta_i P_{s,a}^{i}V_{h+1}^k \} \\
		&\geq  \min \{1, r(s,a)+ \sum_{i=1}^d \vtheta^*_i P_{s,a}^{i}V_{h+1}^k \} \geq \min \{1, r(s,a)+ \sum_{i=1}^d \theta^*_i P_{s,a}^{i}V_{h+1}^* \} = Q^*_h(s,a), \\
	V_h^k(s) &= \max_a Q_h^k(s,a)\geq \max_a Q_h^*(s,a) = V^*_h(s).
	\end{align*}
	
We now prove the second inequality. Let $\delta' = e^{-\iota}.$ We define the desired event $\gE = \bigcap_{k, m, i, j} \gE_k^{m,i,j},$ where 
\begin{align*}
    \gE_k^{m,i,j} = \Bigg\{\abs{\sum_{(v,u)\in \gT^{m,i}_k} \pj_j(\vx_{v, u}^m \vmu) \varepsilon_{\kappa,h}^m}  \leq 4\sqrt{ \sum_{(v,u)\in \gT^{m,i}_k}  \pj_j^2(\vx_{v,u}^{m} \vmu)\Var(\varepsilon_{v,u}^m \mid \gF_u^v)   \ln \frac{1}{\delta'}} +4\ell_{j}\ln \frac{1}{\delta'}, \forall \vmu \in \gB\Bigg\}.
\end{align*}

Note that for a fixed $k,$ we have that $\lvert \pj_j(\vx_{v,u}^m \vmu) \varepsilon_{v,u}^m \rvert \le \ell_j \le 1$ and that 
\begin{align*}
    \Var\left(\pj_j(\vx_{v,u}^m \vmu) \varepsilon_{v,u}^m \sI\{(v,u)\in \gT^{m,i}_k\} \mid \gF_u^v\right) = \pj_j(\vx_{k, h}^m \vmu)^2\sI\{(v,u)\in \gT^{m,i}_k\}  \Var(\varepsilon_{v,u}^m \mid \gF_u^v),
\end{align*} so by Lemma~\ref{lem:ten} with $b = \ell_j, \epsilon =1$, we have 
\begin{align*}
    &\quad \Pr[\abs{\sum_{(v,u)\in \gT^{m,i}_k} \pj_j(\vx_{v,u}^m \vmu) \varepsilon_{v,u}^m}  \ge 4\sqrt{ \sum_{(v,u)\in \gT^{m,i}_k}  \pj_j^2(\vx_{v,u}^{m} \vmu)\Var(\varepsilon_{v,u}^m \mid \gF_u^v)   \ln \frac{1}{\delta'}} +4\ell_{j}\ln \frac{1}{\delta'}]\\
    &\le 4 \delta' \log_2(HK). 
\end{align*}
Using a union bound over $(\vmu,m,i,j,k) \in \gB \times \Lambda_0 \times \Lambda_1 \times \Lambda_2 \times [K],$ we have $\Pr[\gE] \ge 1 - O(\delta' K \abs{\gB}  \log^4(HK)) \ge 1-O(\delta).$

Next we show that the event $\gE$ implies that $\vtheta^* \in \Theta_k$ for every $k \in [K].$ We show by induction over $k.$ For $k = 1$ it is clear. For $k \ge 1,$ since $\vtheta^* \in \Theta_{k},$ for every $h \in [H],$ we have $\eta_{k, h}^m = \max_{\vtheta \in \Theta_k} \{\vtheta \vx_{k,h}^{m+1} - (\vtheta \vx_{k, h}^m)^2\} \ge \vtheta^* \vx_{k,h}^{m+1} - (\vtheta^* \vx_{k, h}^m)^2 \ge \Var(\varepsilon_{k, h}^m \mid \gF_h^k),$ which, together with the event $\bigcap_{m, i, j} \gE_{k+1}^{m,i,j},$ implies that $\vtheta^* \in \Theta_{k+1}.$
\end{proof}

\subsection{Proof of Lemma~\ref{lem:boundI}}

\label{app:lem-boundI-proof}

\begin{proof} We define 
\begin{align*}
    I_{k,h}^{m,i,j} = \sI\{\forall u \le h :  \Phi_{k,u}^{m,i,j}(\vmu_{k,u}^m) \le 4(d+2)^2 \Phi_k^{m,i,j}(\vmu_{k,u}^m)\}.
\end{align*}
Then we have $I^k_h = \prod_{m,i,j} I_{k,h}^{m,i,j}.$ Also we have
\begin{align*}
\sum_{h} (I^k_h - I^k_{h+1}) \le \sum_{m, i, j} \sum_{h} (I_{k,h}^{m,i,j} - I_{k,h+1}^{m,i,j}).  
\end{align*}
Note that $I^k_h \ge I^k_{h+1}$ and $I_{k,h}^{m,i,j} \ge I_{k,h+1}^{m,i,j}.$ For each fixed $m, i, j,$ if $\sum_h (I_{k,h}^{m,i,j} - I_{k,h+1}^{m,i,j}) = 1,$ then there exists $h \in [H],$ such that for the time step $(k,h),$ we have 
$\Phi_{k,h}^{m,i,j}(\vmu) > 4(d+2)^2 \Phi_k^{m,i,j}(\vmu)$ for some $\vmu.$ By Lemma~\ref{lemma:ener_bound} with $f(x)= \clip(x,\ell_{j})x$ and $\ell = \ell_j$, there are at most $O(d \log^2(dHK))$ such time steps. We conclude by noting that we have $\abs{\Lambda_0 \times \Lambda_1 \times \Lambda_2} \le O(\log^3(dHK))$ possible $m,i,j$ pairs.
\end{proof}

\subsection{Proof of Lemma~\ref{lem:sumci}}
To prove this lemma, we define the index sets to help us apply the peeling technique.
We denote
\begin{align*}
\gT^{m,i,j}_{k} = \{(v,u) \in \gT^{m,i}_k : \abs{\vx_{v,u}^m \vmu_{v,u}^m} \in (\ell_{j+1}, \ell_j]\}, \\\gT^{m,i,L_2 + 1}_k = \{(v,u) \in \gT^{m,i}_k : \abs{\vx_{v,u}^m \vmu_{v,u}^m} \in [0, \ell_{L_2+1}]\},
\end{align*}
and $\cgT^{m,i,j}_k = \{(v,u) \in \gT^{m,i,j}_k : I^v_u = 1\}.$ 
We also denote $\gT^{m,i,j} = \gT^{m,i,j}_{K+1}, \cgT^{m,i,j} = \cgT^{m,i,j}_{K+1}.$

\label{app:lem-sumci-proof}

\begin{proof} Since $\vtheta_{k,h}^m \in \Theta_k \subseteq \Theta_k^{m,i,j},$ choosing $\vmu = \vmu_{k, h}^m$ in the confidence set definition and using that $\vx_{v, u}^m \vmu_{k, h}^m = \epsilon_{v,u}^m(\vtheta^*) - \epsilon_{v, u}^m(\vtheta_{k, h}^m),$ we have
	\begin{align}
	\Phi_{k}^{m, i, j}(\vmu_{k,h}^m) &=    \sum_{(v, u) \in \gT^{m,i}_k} \pj_j(\vx^m_{v,u} \vmu^m_{k,h}) \vx^m_{v,u} \vmu^m_{k,h} + \ell_j^2 \notag \\
	&\le \abs{\sum_{(v, u) \in \gT^{m,i}_k}  \pj_j(\vx_{v,u}^m \vmu_{k,h}^m) \epsilon_{v,u}^m(\vtheta^*)} + \abs{\sum_{(v, u) \in \gT^{m,i}_k}  \pj_j(\vx_{v,u}^m \vmu_{k,h}^m) \epsilon_{v,u}^m(\vtheta^m_{k,h})} + \ell_j^2 \notag \\
	&\le 8 \sqrt{\sum_{(v, u) \in \gT^{m,i}_k}  \pj_j(\vx_{v,u}^m \vmu_{k,h}^m) \eta_{v,u}^m\iota} + 8 \ell_j \iota + \ell_j^2 \notag \\
	&\le 8 \sqrt{\Psi^{m,i,j}_{k}(\vmu_{k,h}^m)\iota} + 16 \ell_j \iota.
	\end{align}
	Therefore, when $I_k^h = 0,$ we have 
	\begin{align*}
	\frac{\Phi_{k, h}^{m,i,j}(\vmu_{k,h}^m)}{4(d+2)^2} \le \Phi_{k}^{m, i, j}(\vmu_{k,h}^m) \le 16( \sqrt{\Psi^{m,i,j}_{k,h}(\vmu_{k,h}^m)\iota} +  \ell_j \iota ).
	\end{align*}
	Next we analyze the sum.
	Using the fact that \[\frac{64(d+2)^2  \left(\sqrt{\Psi^{m,i,j}_{k,h}(\vmu_{k,h}^m)\iota} +  \ell_j \iota\right)}{\Phi_{k, h}^{m,i,j}(\vmu_{k,h}^m)} \ge 1, \]
	we obtain
	\begin{align}
	\sum_{(k,h)\in \cgT^{m,i,j}} \left| \vx_{k,h}^m \vmu_{k,h}^m \right| &\leq \sum_{(k,h)\in \cgT^{m,i,j}}   \left| \vx_{k,h}^m \vmu_{k,h}^m \right|  \frac{64(d+2)^2  \left(\sqrt{\Psi^{m,i,j}_{k,h}(\vmu_{k,h}^m)\iota} +  \ell_j \iota\right)}{\Phi_{k, h}^{m,i,j}(\vmu_{k,h}^m)}
	\\ & \leq 64(d+2)^2 \sum_{(k,h)\in \cgT^{m,i,j}} \left( \frac{ \left|\vx_{k,h}^m \vmu_{k, h}^m \right|  \sqrt{\ell_i\iota}}{\sqrt{\Phi_{k,h}^{m,i,j}(\vmu_{k,h}^m)} }  + \frac{\left|\vx_{k,h}^m \vmu_{k, h}^m \right| \ell_j\iota}{\Phi_{k,h}^{m,i,j}(\vmu_{k,h}^m)}\right), \label{lem5-1000}
	\end{align}
	where the last inequality uses that for every $\vmu,$ we have 
	\begin{align}
	\Psi_{k,h}^{m,i,j}(\vmu) = \sum_{(v,u)\in\gT^{m,i}_{k,h}} \pj_j^2(\vx_{v,u}^m \vmu) \eta_{v,u}^m \le  \ell_i \sum_{(v,u)\in\gT^{m,i}_{k,h}} \pj_j(\vx_{k,h}^m \vmu) \vx_{k,h}^m \vmu \le \ell_i \Phi_{k,h}^{m,i,j}(\vmu). \label{eqn:rl_psi}
	\end{align}
	In \eqref{eqn:rl_psi}, the first inequality uses that $\eta_{v,u}^m \le \ell_i$ for $(v, u) \in \gT^{m,i}_{k,h}$ and that $\pj_j^2(\alpha)\le \pj_j(\alpha) \alpha$ for $\alpha \in \sR,$ and the second inequality uses the definition of $\Phi_{k,h}^{m,i,j}(\vmu).$ Next we bound the two terms in \eqref{lem5-1000}. To bound the first term, we note that 
	\begin{align}
	\hspace{-2em}   \sum_{(k,h)\in \cgT^{m,i,j}}  \frac{\left|\vx_{k,h}^m \vmu_{k, h}^m  \right|}{\sqrt{\Phi_{k,h}^{m,i,j}(\vmu_{k,h}^m)}} &\le \sqrt{\abs{\cgT^{m,i,j}}} \sqrt{ \sum_{(k,h)\in \cgT^{m,i,j}} \frac{(\vx_{k,h}^m \vmu_{k,h}^m)^2}{\Phi_{k,h}^{m,i,j}(\vmu_{k,h}^m)} } \label{lem5-2000} \\
	&\le  \sqrt{\abs{\cgT^{m,i,j}}} \sqrt{ \sum_{(k,h)\in \cgT^{m,i,j}} \frac{\pj_j^2(\vx_{k,h}^m \vmu_{k,h}^m)}{\Phi_{k,h}^{m,i,j}(\vmu_{k,h}^m)} } \label{lem5-2100}\\
	&\le \sqrt{\abs{\cgT^{m,i,j}}}   \sqrt{ \sum_{(k,h)\in \cgT^{m,i,j}} \frac{\pj_j^2(\vx_{k,h}^m \vmu_{k,h}^m)}{\sum\limits_{(v,u) \in \cgT^{m,i,j}_{k,h}}  \pj_j(\vx^m_{v,u}\vmu_{k,h}^m) \vx^m_{v,u}\vmu_{k,h}^m + \ell_j^2} } \label{lem5-2200}\\
	&\le \sqrt{\abs{\cgT^{m,i,j}}} \times O(\sqrt{d^4  \log^3(dHK)}), \label{lem5-2300}
	\end{align}
	where \eqref{lem5-2000} uses Cauchy's inequality, \eqref{lem5-2100} uses that $\left|\vx_{k,h}^m \vmu_{k,h}^m\right| \le \ell_j$ for $(k, h) \in \gT^{m,i,j},$ \eqref{lem5-2200} uses the definition of $\Phi^{m,i,j}_{k,h}(\vmu),$ and \eqref{lem5-2300} uses Lemma~\ref{lemma:bound_m2}. To bound the second term in \eqref{lem5-1000}, we have 
	\begin{align}
	\sum_{(k,h)\in \cgT^{m,i,j}}  \frac{ \left|\vx_{k,h}^m \vmu_{k, h}^m \right| \ell_j }{\Phi_{k,h}^{m,i,j}(\vmu_{k,h}^m)} \le \sum_{(k,h)\in \gT^{m,i,j}} \frac{2\pj_j^2(\vx_{k,h}^m \vmu_{k,h}^m)}{\Phi_{k,h}^{m,i,j}(\vmu_{k,h}^m)} \le O(d^4 \log^3(dHK)), \label{lem5-3000}
	\end{align}
	where the first inequality uses that $\left|\vx_{k,h}^m \vmu_{k,h}^m\right| \ge \ell_j/2$ for $(k, h) \in \gT^{m,i,j}$ and the second inequality is the same as what we have shown from \eqref{lem5-2100} to \eqref{lem5-2300}. As a result, combining \eqref{lem5-1000},\eqref{lem5-2300} and \eqref{lem5-3000}, we have 
	\begin{align}
	\sum_{(k,h)\in \cgT^{m,i,j}} \left| \vx_{k,h}^m \vmu_{k,h}^m \right| &\le 64(d+2)^2 \times O\left( \sqrt{d^4 \ell_i \abs{\cgT^{m,i,j}}  \iota \log^3(dHK)} + {d^4 \iota \log^3(dHK)} \right) \\ 
	&\le O\left(d^4\sqrt{\ell_i \abs{\cgT^{m,i,j}}  \iota \log^3(dHK)} + {d^6 \iota \log^3(dHK)} \right). \label{lem5-4000}
	\end{align}

	Recall that \eqref{lem5-4000} requires $\vx_{k,h}^m \vmu_{k,h}^m \in [\ell_j/2,\ell_j],$ which would be false for $j = L_2 + 1.$ In this corner case,  $j = L_2 + 1,$ we have 
	\begin{align}
	\sum_i \sum_{(k,h)\in \cgT^{m, i,j}} \left|\vx_{k,h}^m \vmu_{k,h}^m \right| \le KH \ell_j \le O(1). \label{lem5-5000}
	\end{align}
	Finally, combining \eqref{lem5-4000} and \eqref{lem5-5000}, we have 
	\begin{align}
	\sum_{k, h}\left| \cvx_{k,h}^m \vmu_{k,h}^m \right| &= \sum_{i, j} \sum_{(k, h) \in \cgT^{m, i, j}} \left|\vx_{k,h}^m \vmu_{k,h}^m \right|\notag \\
	&\le O(1) + \sum_{i,j}  O\left(d^4\sqrt{\ell_i \abs{\cgT^{m,i,j}}  \iota \log^3(dHK)} + L_2 {d^6 \iota \log^3(dHK)} \right) \notag \\
	&\le O\left(d^4\sqrt{\sum_{k, h} \ceta_{k,h}^m  \iota \log^7(dHK)} +  {d^6 \iota \log^5(dHK)} \right), \label{lem5-6000}
	\end{align}
	where \eqref{lem5-6000} uses that $\ell_i \abs{\cgT^{m,i,j}} \le O(1 + \sum_{k, h} \ceta_{k,h}^m),$ which can be proved as follows: for $i \le L_1,$ it is due to $\eta_{k,h}^m \ge \ell_i/2$; for $i = L_1 + 1,$ it is due to $1/\ell_i \ge KH \ge \abs{\cgT^{m,i,j}}.$
\end{proof}

\subsection{Proof of Lemma~\ref{lem:recurrsion}}

\label{app:lem-recurrsion-proof}

\begin{proof} Define  
	\begin{align*}
	\czeta_{k,h}^m = (P_{s_h^k, a_h^k}(V_{h + 1}^k)^{2^{m + 1}} - (P_{s_h^k, a_h^k}(V_{h + 1}^k)^{2^{m}})^2) I^k_h.
	\end{align*}
	We note that $\cM_m$ is a martingale, so  by Lemma~\ref{lem:ten} with a union bound over $m,$ we have 
	\begin{align}
	\Pr[\forall m\in \Lambda_0 : \abs{\cM_m} \le 2 \sqrt{2 \sum_{k,h} \czeta_{k,h}^m \ln \frac{1}{\delta}} + 4\ln\frac{1}{\delta}]\ge 1- O(\delta \log^2(dKH)). \label{eq:recur-3000}
	\end{align}
	By the definition of $\ceta_{k,h}^m,$ we have 
	\begin{align}
	\sum_{k,h}\ceta_{k,h}^m \leq &\sum_{k,h} \left( \czeta_{k,h}^m + \max_{\vtheta\in \Theta_k} \cvx_{k,h}^{m+1}(\vtheta-\vtheta^*) +2 \max_{\vtheta\in \Theta_k} \cvx_{k,h}^m(\vtheta^*-\vtheta) \right) \\
	\leq &\sum_{k,h}  \czeta_{k,h}^m + \cR_{m+1} + 2\cR_m,\label{eq:recur-2000}
	\end{align}
	We have that 
	\begin{align}
	\sum_{k,h}  \czeta_{k,h}^m &= \sum_{k,h} \left( P_{s_h^k, a_h^k}(V_{h + 1}^k)^{2^{m + 1}} - (P_{s_h^k, a_h^k}(V_{h + 1}^k)^{2^{m}})^2\right)I_h^k \notag \\ 
	&\le \sum_{k,h} \left( P_{s_h^k, a_h^k}(V_{h + 1}^k)^{2^{m + 1}} - (V_{h + 1}^k(s_{h+1}^k))^{2^{m + 1}}\right) I_h^k  + \sum_{k, h} (V_{h}^k(s_h^k))^{2^{m + 1}} (I_h^k - I^k_{h+1}) \notag \\
	&\quad +\sum_{k,h} \left((V_{h}^k(s_h^k))^{2^{m + 1}} - (P_{s_h^k, a_h^k}(V_{h + 1}^k)^{2^{m}})^2 \right) I_h^k \notag \\
	&\le \cM_{m+1} + O(d \log^5(dHK)) + \sum_{k,h} \left((V_{h}^k(s_h^k))^{2^{m + 1}} - (P_{s_h^k, a_h^k}(V_{h + 1}^k)^{2^{m}})^2 \right) I_h^k \notag \\
	&\le \cM_{m+1} + O(d \log^5(dHK)) + \sum_{k,h} \left((V_{h}^k(s_h^k))^{2^{m + 1}} - (P_{s_h^k, a_h^k} V_{h + 1}^k)^{2^{m+1}} \right) \notag \\
	&\le \cM_{m+1} + O(d \log^5(dHK)) +  2^{m + 1} \sum_{k,h} I^k_h \cdot \max\{V_h^k(s_h^k) -  P_{s_h^k, a_h^k} V_{h + 1}^k, 0\} \notag \\
	&\le\cM_{m+1} + O(d \log^5(dHK)) +  2^{m + 1} \sum_{k,h} I^k_h \left(r(s_h^k, a_h^k) +\max_{\vtheta \in \Theta_k} \vx_{k,h}^0 (\vtheta - \vtheta^*) \right)\notag \\
	&\le \cM_{m+1} + O(d \log^5(dHK)) +  2^{m + 1}(K + \cR_0). \label{eq:recur-4000}
	\end{align}
	Finally, by \eqref{eq:recur-2000}, \eqref{eq:recur-4000}  and Lemma~\ref{lem:sumci}, we have
	\begin{align}
	\hspace{-5em} \cR_m &\le O\left(d^4\sqrt{( \cM_{m+1} + O(d \log^5(dHK)) +  2^{m + 1}(K + \cR_0) +  \cR_{m+1} + 2\cR_m)  \iota \log^7(dHK)} +  {d^6 \iota \log^5(dHK)} \right) \notag \\
	&\le O\left(d^4\sqrt{( \cM_{m+1}+  2^{m + 1}(K + \cR_0) +  \cR_{m+1} + \cR_m)  \iota \log^7(dHK)} +  {d^6 \iota \log^5(dHK)} \right), \label{eq:recur-6000}
	\end{align}
	which proves the first part of the lemma. By \eqref{eq:recur-3000} and \eqref{eq:recur-4000}, we have
	\begin{align}
	\abs{\cM_m} \le O\left(\sqrt{(\cM_{m+1} + O(d \log^5(dHK)) +  2^{m + 1}(K + \cR_0)) \log(1/\delta)} + \log(1/\delta)\right), \label{eq:recur-7000}
	\end{align}
	which proves the second part of the lemma.
\end{proof}

\subsection{Proof of Lemma~\ref{lem:cr0_cm0}}

\label{app:lem-cr0-cm0-proof}

\begin{proof}
	Let $b_m = \cR_m + \lvert{\cM_m}\rvert.$ By \eqref{eq:recur-6000} and \eqref{eq:recur-7000}, we can bound $b_m$ recursively as 
	\begin{align}
	b_m \leq O\left( \sqrt{d^9\log^{5}(Td) \log \frac{1}{\delta}} \sqrt{b_m + b_{m+1}+2^{m + 1}(K + \cR_0) }+d^7\log^{6}(Td)\log\frac{1}{\delta} \right).\label{eq:efe6}
	\end{align}
	Note that $b_m\leq 2KH$ for $m \in \Lambda_1$. By Lemma~\ref{lemma:sequence2} with parameters
	\begin{align*}
	\lambda_1 = 2 KH, \quad \lambda_2 =  \sqrt{d^9\log^5(Td)\log(1/\delta)}, \quad \lambda_3 = K + \cR_0, \quad \lambda_4 = d^7\log^6(Td) \log(1/\delta),
	\end{align*}
	we obtain that 
	\begin{align*}
	\cR_0 \leq b_0 \le O\left(   \sqrt{d^9(K+\cR_0)\log^5(Td)\log(1/\delta)} +d^9\log^6(Td)\log(1/\delta)\right),
	\end{align*}
	which implies 
	\begin{align*} b_0 \le O\left( d^{4.5}\sqrt{ K \log^5(Td)\log(1/\delta)}+d^9\log^6(Td)\log(1/\delta)\right)
	\end{align*}
	and completes the proof.
\end{proof}

\end{document}